\documentclass[runningheads]{llncs}
\pdfoutput=1
\usepackage{graphicx}

\usepackage{amsmath}
\usepackage{amssymb}
\usepackage{mathbbol}
\usepackage{bbm}
\DeclareSymbolFontAlphabet{\amsmathbb}{AMSb}
\usepackage{mathtools}

\usepackage{hyperref}
\usepackage{cleveref}

\usepackage[noadjust,nospace]{cite}

\usepackage{amsmath,amssymb}       
\usepackage{nicefrac}
\usepackage{mleftright}
\usepackage{cases}
\usepackage{empheq}
\usepackage{mathbbol}
\usepackage{stmaryrd}
\usepackage{bm}
\usepackage{bbm}
\DeclareSymbolFontAlphabet{\amsmathbb}{AMSb}
\usepackage{mathtools}
\usepackage{accents}
\usepackage[colorinlistoftodos]{todonotes}

\usepackage{siunitx}
\usepackage{caption}
\usepackage{subcaption}

\graphicspath{./Figures/}

\newcommand\munderbar[1]{%
  \underaccent{\bar}{#1}}

\usepackage{xspace}
\usepackage{microtype}

\usepackage{cleveref}

\crefname{equation}{Eq.}{Eqs.}
\crefname{pluralequation}{Eqs.}{Eqs.}

\crefname{algorithm}{Algorithm}{Algorithm}

\crefname{figure}{Fig.}{Figs.}
\crefname{pluralfigure}{Figs.}{Figs.}

\crefname{section}{Sect.}{Sects.}
\crefname{pluralsection}{Sects.}{Sects.}

\crefname{table}{Table}{Table}
\crefname{pluraltable}{Tables}{Tables}

\crefname{definition}{Def.}{Def.}
\crefname{pluraldefinition}{Defs.}{Defs.}

\crefname{theorem}{Theorem}{Theorems}
\crefname{pluraltheorem}{Theorems}{Theorems}

\crefname{lemma}{Lemma}{Lemmas}
\crefname{plurallemma}{Lemmas}{Lemmas}

\crefname{example}{Example}{Example}
\crefname{pluralexample}{Examples}{Examples}

\crefname{problem}{Problem}{Problem}
\crefname{pluralproblem}{Problems}{Problems}

\crefname{assumption}{Assumption}{Assumption}
\crefname{pluralassumption}{Assumptions}{Assumptions}

\crefname{remark}{Remark}{Remark}
\crefname{pluralremark}{Remarks}{Remarks}

\crefname{proposition}{Proposition}{Proposition}
\crefname{pluralproposition}{Propositions}{Propositions}

\crefname{appendix}{Appendix}{Appendices}
\crefname{pluralappendix}{Appendices}{Appendices}

\usepackage{tikz} 
\usepackage{pgfplots}
\usepgfplotslibrary{external,fillbetween}
\pgfplotsset{compat=1.8}

\definecolor{red}{rgb}{0.745,0.192,0.102}
\definecolor{darkgreen}{RGB}{34,161,55}
\definecolor{ruhuisstijlrood}{rgb}{0.745,0.192,0.102}
\definecolor{ruhuisstijlzwart}{rgb}{0,0,0}
\definecolor{ruhuisstijlwit}{rgb}{0.98,0.98,0.98}

\usepackage{mdframed}

\usetikzlibrary{patterns,patterns.meta,arrows,arrows.meta,calc,shapes,shadows,decorations.pathmorphing,decorations.pathreplacing,automata,shapes.multipart,positioning,shapes.geometric,fit,circuits,trees,shapes.gates.logic.US,fit, shadows.blur,shapes.symbols,intersections}

\usepackage{fontawesome5}

\usepackage{tabularx} 
\usepackage{enumitem}

\usepackage{tabularx}
\usepackage{booktabs}
\usepackage{scrextend}
\usepackage{threeparttable}
\newcommand{\ie}{i.e.\@}

\newcommand{\eg}{e.g.\xspace}

\newcommand{\storm}{\textrm{Storm}}

\newcommand{\U}{\mathbf{u}}

\newcommand{\Param}{\ensuremath{V}}
\newcommand{\param}{\ensuremath{v}}

\DeclareMathOperator*{\minimize}{minimize}
\DeclareMathOperator*{\maximize}{maximize}

\newcommand{\Real}{\amsmathbb{R}}
\newcommand{\Q}{\amsmathbb{Q}}

\newcommand{\F}{\amsmathbb{T}}
\newcommand{\N}{\amsmathbb{N}}
\newcommand{\distr}[1]{\ensuremath{\mathit{Dist(#1)}}}
\newcommand{\pdistr}[1]{\ensuremath{\mathit{pDist_{\Param}(#1)}}}
\newcommand{\pfun}[1]{\ensuremath{\mathit{pFun_{\Param}(#1)}}}

\newcommand{\States}{\ensuremath{S}}

\newcommand{\terminalStates}{\ensuremath{S}_T}

\newcommand{\sI}{\ensuremath{s_I}}
\newcommand{\transfunc}{\ensuremath{P}}
\newcommand{\transfuncImdp}{\ensuremath{\mathcal{P}}}

\newcommand{\PMC}{\ensuremath{(\States,\sI,\Param,\transfunc)}}
\newcommand{\RMC}{\ensuremath{(\States,\sI,\transfuncImdp)}}
\newcommand{\PRMC}{\ensuremath{(\States,\sI,\Param,\transfuncImdp)}}

\newcommand{\pmc}{\ensuremath{\mathcal{M}}}
\newcommand{\prmc}{\ensuremath{\mathcal{M}_{R}}}

\newcommand{\diag}[1]{\ensuremath{{\mathbf{D}(#1)}}}

\newcommand{\ExpR}{\ensuremath{\mathsf{sol}}}
\newcommand{\ExpRR}{\ensuremath{\mathsf{sol}_R}}

\newcommand{\post}[1]{\ensuremath{{\mathsf{post}(#1)}}}

\newcommand{\parder}[2]{\ensuremath{{\frac{\partial #1}{\partial #2}}}}
\newcommand{\grad}[2]{\ensuremath{{\nabla #2 }}}

\usepackage[acronym]{glossaries}

\glsdisablehyper

\newacronym[]{LTL}{LTL}{linear temporal logic}
\newacronym[]{iid}{i.i.d.}{independent and identically distributed}
\newacronym[]{UAV}{UAV}{unmanned aerial vehicle}

\newacronym[]{PAC}{PAC}{probably approximately correct}
\newacronym[]{DRO}{DRO}{distributionally robust optimization}
\newacronym[plural=MCs,firstplural=Markov chains (MCs)]{MC}{MC}{Markov chain}
\newacronym[plural=DTMCs,firstplural=discrete-time Markov chains (DTMCs)]{DTMC}{DTMC}{discrete-time Markov chain}
\newacronym[plural=MDPs,firstplural=Markov decision processes (MDPs)]{MDP}{MDP}{Markov decision process}
\newacronym[plural=pMCs,firstplural=parametric MCs (pMCs)]{pMC}{pMC}{parametric MC}
\newacronym[plural=rMCs,firstplural=robust MCs (rMCs)]{rMC}{rMC}{robust MC}
\newacronym[plural=prMCs,firstplural=parametric robust MCs (prMCs)]{prMC}{prMC}{parametric robust MC}
\newacronym[plural=prMDPs,firstplural=parametric robust MDPs (prMDPs)]{prMDP}{prMDP}{parametric robust MDP}
\newacronym[plural=iMDPs,firstplural=interval Markov decision processes (iMDPs)]{iMDP}{iMDP}{interval Markov decision process}
\newacronym[plural=aMDPs,firstplural=augmented MDPs (aMDPs)]{aMDP}{aMDP}{augmented MDP}
\newacronym[plural=POMDPs,firstplural=partially observable Markov decision processes (POMDPs)]{POMDP}{POMDP}{partially observable Markov decision process}

\definecolor{color1}{RGB}{55,126,184} 
\definecolor{color2}{RGB}{228,26,28} 
\definecolor{color3}{RGB}{77,175,74} 
\definecolor{color4}{RGB}{152,78,163} 
\definecolor{color5}{RGB}{255,127,0} 
\definecolor{color6}{rgb}{0.5, 1.0, 0.83} 
\definecolor{color7}{rgb}{1.0, 0.0, 1.0} 
\definecolor{color8}{rgb}{0.66, 0.66, 0.66} 

\newcommand{\scatterplotstorm}[6]{%
	\begin{tikzpicture}
	\begin{axis}[
	width=\scatterplotsize,
	height=\scatterplotsize,
	axis equal image,
	xmin=0.01,
	ymin=0.01,
	ymax=22000,
	xmax=22000,
	xmode=log,
	ymode=log,
	axis x line=bottom,
	axis y line=left,
	xtick={0.01,0.1,1,5,20,100,1000,3000},
	xticklabels={0.01,0.1,1,5,20,100,1000,3000},
	extra x ticks = {10000},
	extra x tick labels = {Timeout},
	extra x tick style = {grid = major},
	ytick={0.01,0.1,1,5,20,100,1000,3000},
	yticklabels={0.01,0.1,1,5,20,100,1000,3000},
	extra y ticks = {10000},
	extra y tick labels = {Timeout},
	extra y tick style = {grid = major},
	xlabel={#3},
	xlabel style={font=\small,yshift=18pt,xshift=-14pt},
	ylabel={#5},
	ylabel style={font=\small,yshift=-0.55cm},
	yticklabel style={font=\tiny},
	xticklabel style={rotate=290,anchor=west,font=\tiny},
	legend pos=north west,
	legend columns=-1,
	legend style={at={(0.4,0.15)},nodes={scale=0.75, transform shape},inner sep=1.5pt},
        clip mode=individual,
	]
	
	\addplot[
	scatter,
	only marks,
	scatter/classes={
		pMC={mark=*,color1,mark size=1.5},
		prMC={mark=triangle*,color2,mark size=1.75}
	},
	scatter src=explicit symbolic
	]%
	table [col sep=semicolon,x=#2,y=#4,meta=Type] {#1};
	\ifthenelse{\NOT\equal{#6}{false}}{\legend{pMC, prMC}}{}
	\addplot[no marks] coordinates {(0.001,0.001) (10000,10000) };
	\addplot[no marks, densely dotted] coordinates {(0.001,0.01) (1000,10000)};
	\addplot[no marks, densely dotted] coordinates {(0.01,0.001) (10000,1000)};

        \draw [latex-] (axis cs:50,500)-- +(-6pt,5pt) node[left, xshift=5pt, yshift=5pt] {$10\times$ faster};
 
	\end{axis}
	\end{tikzpicture}
}

\newcommand{\scatterplotstormB}[6]{%
	\begin{tikzpicture}
	\begin{axis}[
	width=\scatterplotsize,
	height=\scatterplotsize,
	axis equal image,
	xmin=0.001,
	ymin=0.001,
	ymax=22000,
	xmax=22000,
	xmode=log,
	ymode=log,
	axis x line=bottom,
	axis y line=left,
	xtick={0.01,0.1,1,5,20,100,1000,3000},
	xticklabels={0.01,0.1,1,5,20,100,1000,3000},
	extra x ticks = {10000},
	extra x tick labels = {Timeout},
	extra x tick style = {grid = major},
	ytick={0.01,0.1,1,5,20,100,1000,3000},
	yticklabels={0.01,0.1,1,5,20,100,1000,3000},
	extra y ticks = {10000},
	extra y tick labels = {Timeout},
	extra y tick style = {grid = major},
	xlabel={#3},
	xlabel style={font=\small,yshift=18pt,xshift=-5pt},
	ylabel={#5},
	ylabel style={font=\small,yshift=-0.55cm,xshift=-0.2cm},
	yticklabel style={font=\tiny},
	xticklabel style={rotate=290,anchor=west,font=\tiny},
	legend pos=north west,
	legend columns=-1,
	legend style={at={(0.4,0.15)},nodes={scale=0.75, transform shape},inner sep=1.5pt},
        clip mode=individual,
	]
	
	\addplot[
	scatter,
	only marks,
	scatter/classes={
		pMC={mark=*,color1,mark size=1.5},
		prMC={mark=triangle*,color2,mark size=1.75}
	},
	scatter src=explicit symbolic
	]%
	table [col sep=semicolon,x=#2,y=#4,meta=Type] {#1};
	\ifthenelse{\NOT\equal{#6}{false}}{\legend{pMC, prMC}}{}
	\addplot[no marks] coordinates {(0.001,0.001) (10000,10000) };
	\addplot[no marks, densely dotted] coordinates {(0.001,0.01) (1000,10000)};
	\addplot[no marks, densely dotted] coordinates {(0.01,0.001) (10000,1000)};

        \draw [latex-] (axis cs:100,1000)-- +(-6pt,8pt) node[left] {$10\times$ faster};
 
	\end{axis}
	\end{tikzpicture}
}

\newcommand{\tableextrapolationtext}{Extrapolated from the runtimes for $10$ to all $|\Param|$ parameters.}
\newcommand{\tabletimeouttext}{Timeout (1 hour) occurred for verifying the p(r)MC, not for computing derivatives.}

\spnewtheorem{assumption}{Assumption}{\bfseries}{\itshape}
\spnewtheorem*{problemForm*}{Formal problem}{\bfseries}{}

\renewcommand{\paragraph}[1]{\smallskip\noindent\emph{#1}}
\renewcommand{\subsubsection}[1]{\smallskip\noindent\textbf{#1}}

\newif\iftikzcompile
\tikzcompiletrue

\usetikzlibrary{external}

\begin{document}

\title{
Efficient Sensitivity Analysis for \\ Parametric Robust Markov Chains
\thanks{This research has been partially funded by NWO grant NWA.1160.18.238 (PrimaVera), the ERC Starting Grant 101077178 (DEUCE), and grants ONR N00014-21-1-2502 and AFOSR FA9550-22-1-0403.}
}
\author{ 
Thom Badings\inst{1}
\and
Sebastian Junges\inst{1}
\and
Ahmadreza Marandi\inst{2}
\and\\
Ufuk Topcu\inst{3}
\and
Nils Jansen\inst{1}
}

\authorrunning{T. Badings et al.}

 \institute{
Radboud University, Nijmegen, the Netherlands
\\
\email{thom.badings@ru.nl} 
\and
Eindhoven University of Technology, the Netherlands
\and
University of Texas at Austin, USA
}
\maketitle              
\begin{abstract}
We provide a novel method for sensitivity analysis of parametric robust Markov chains. These models incorporate parameters and sets of probability distributions to alleviate the often unrealistic assumption that precise probabilities are available. We measure sensitivity in terms of partial derivatives with respect to the uncertain transition probabilities regarding measures such as the expected reward. As our main contribution, we present an efficient method to compute these partial derivatives. To scale our approach to models with thousands of parameters, we present an extension of this method that selects the subset of $k$ parameters with the highest partial derivative. Our methods are based on linear programming and differentiating these programs around a given value for the parameters. The experiments show the applicability of our approach on models with over a million states and thousands of parameters. Moreover, we embed the results within an iterative learning scheme that profits from having access to a dedicated sensitivity analysis.
\end{abstract}

\setcounter{footnote}{0} 

\section{Introduction}
\label{sec:introduction}

Discrete-time \glspl{MC} are ubiquitous in stochastic systems modeling~\cite{DBLP:books/daglib/BaierKatoen2008}.
A classical assumption is that all probabilities of an MC are precisely known---an assumption that is difficult, if not impossible, to satisfy in practice~\cite{Badings2023_STTT_position}.
\Glspl{rMC}, or uncertain \glspl{MC}, alleviate this assumption by using \emph{sets of probability distributions},
\eg, intervals of probabilities in the simplest case~\cite{DBLP:conf/lics/JonssonL91,DBLP:books/degruyter/Ben-TalGN09}.
A typical verification problem for \glspl{rMC} is to compute upper or lower bounds on measures of interest, such as the expected cumulative reward, under \emph{worst-case realizations} of these probabilities in the set of distributions~\cite{DBLP:conf/cdc/WolffTM12,DBLP:conf/cav/PuggelliLSS13}.
Thus, verification results are \emph{robust} against any selection of probabilities in these sets.

\paragraph{Where to improve my model?}
As a running example, consider a ground vehicle navigating toward a target location in an environment with different terrain types.
On each terrain type, there is some probability that the vehicle will slip and fail to move.
Assume that we obtain a sufficient number of \emph{samples} to infer upper and lower bounds (i.e., intervals) on the slipping probability on each terrain.
We use these probability intervals to model the grid world as an \gls{rMC}.
However, from the \gls{rMC}, it is unclear how our model (and thus the measure of interest) will change if we obtain more samples.
For instance, if we take one more sample for a particular terrain, some of the intervals of the \gls{rMC} will change, but how can we expect the verification result to change? 
And if the verification result is unsatisfactory, for which terrain type should we obtain more samples?

\paragraph{Parametric robust MCs.}
To reason about how additional samples will change our model and thus the verification result, we employ a sensitivity analysis~\cite{DBLP:journals/tse/FilieriTG16}.
To that end, we use \glspl{prMC}, which are \glspl{rMC} whose sets of probability distributions are defined as a function of a set of \emph{parameters}~\cite{DBLP:conf/syncop/Delahaye15}, \eg, intervals with parametric upper/lower bounds.
With these functions over the parameters, we can describe dependencies between the model's states.
The assignment of values to each of the parameters is called an \emph{instantiation}.
Applying an instantiation to a \gls{prMC} induces an \gls{rMC} by replacing each occurrence of the parameters with their assigned values.
For this induced \gls{rMC}, we compute a (robust) value for a given measure, and we call this verification result the \emph{solution} for this instantiation.
Thus, we can associate a \gls{prMC} with a function, called the \emph{solution function}, that maps parameter instantiations to values.

\paragraph{Differentation for prMCs.}
For our running example, we choose the parameters to represent the number of samples we have obtained for each terrain.
Naturally, the \emph{derivative of this solution function} with respect to each parameter (a.k.a. sample size) then corresponds to the expected change in the solution upon obtaining more samples.
Such differentiation for \glspl{pMC}, where parameter instantiations yield one precise probability distribution, has been studied in~\cite{DBLP:conf/vmcai/HeckSJMK22}.
For \glspl{prMC}, however, it is unclear how to compute derivatives and under what conditions the derivative exists.
We thus consider the following problem:
\begin{enumerate}[label=Problem \arabic*, leftmargin=*]
    \item \textit{(Computing derivatives)}. Given a \gls{prMC} and a parameter instantiation, compute the partial derivative of the solution function (evaluated at this instantiation) with respect to each of the parameters.
\end{enumerate}
\paragraph{Our approach.}
We compute derivatives for \glspl{prMC} by solving a parameterized linear optimization problem.
We build upon results from convex optimization theory for differentiating the optimal solution of this optimization problem~\cite{DBLP:books/cu/BV2014,barratt2018differentiability}.
We also present sufficient conditions for the derivative to exist.

\paragraph{Improving efficiency.}
However, computing the derivative for every parameter explicitly does not scale to more realistic models with thousands of parameters.
Instead, we observe that to determine for which parameter we should obtain more samples, we do not need to know \emph{all partial derivatives explicitly}.
Instead, it may suffice to know which parameters have \emph{the highest} (or lowest, depending on the application) derivative.
Thus, we also solve the following (related) problem:
\begin{enumerate}[label=Problem \arabic*, leftmargin=*]
    \setcounter{enumi}{1}
    \item \textit{($k$-highest derivatives)}. Given a \gls{prMC} with $|\Param|$ parameters, determine the $k < |\Param|$ parameters with the highest (or lowest) partial derivative.
\end{enumerate}
We develop novel and efficient methods for solving Problem 2.
Concretely, we design a linear program (LP) that finds the $k$ parameters with the highest (or lowest) partial derivative without computing all derivatives explicitly.
This LP constitutes a polynomial-time algorithm for Problem~2 and is, in practice, \emph{orders of magnitude faster} than computing all derivatives explicitly, especially if the number of parameters is high.
Moreover, if the concrete values for the partial derivatives are required, one can additionally solve Problem~1 for only the resulting $k$ parameters.
In our experiments, we show that we can compute derivatives for models with over a million states and thousands of parameters.

\paragraph{Learning framework.}
Learning in stochastic environments is very data-intensive in general, and millions of samples may be required to obtain sufficiently tight bounds on measures of interest~\cite{kakade2003sample,Moerland2020modelbasedRL}. Several methods exist to obtain intervals on probabilities based on sampling, including statistical methods such as Hoeffding's inequality~\cite{DBLP:books/daglib/Boucheron2013} and Bayesian methods that iteratively update intervals~\cite{Suilen2022Neurips}.
Motivated by this challenge of reducing the sample complexity of learning algorithms, we embed our methods in an iterative learning scheme that profits from having access to sensitivity values for the parameters.
In our experiments, we show that derivative information can be used effectively to guide sampling when learning an unknown Markov chain with hundreds of parameters.

\paragraph{Contributions.}
Our contributions are threefold:
(1)~We present a first algorithm to compute partial derivatives for \glspl{prMC}.
(2)~For both \glspl{pMC} and \glspl{prMC}, we develop an efficient method to determine a subset of parameters with the highest derivatives.
(3)~We apply our methods in an iterative learning scheme.
We give an overview of our approach in \cref{sec:motivation} and formalize the problem statement in \cref{sec:problem}.
In \cref{sec:differentiating_pMCs}, we solve Problems (1) and (2) for pMCs, and in \cref{sec:differentiating_prMCs} for prMCs.
Finally, the learning scheme and experiments are in \cref{sec:experiments}.
\begin{figure*}[t]
    \centering
    \begin{subfigure}[b]{0.26\textwidth}
        \centering

        \iftikzcompile
        \input{tikz/motivating/grid.tex}
        \fi
        
        \caption{
		Grid world.
	}
        \label{fig:motivating:terrain}
    \end{subfigure}
    \begin{subfigure}[b]{0.40\textwidth}
        \centering

        \resizebox{.85\textwidth}{!}{%
\small
\begin{tabularx}{\linewidth}{llllll}
    \toprule
    & & & & \multicolumn{2}{c}{\textbf{Derivatives}} \\
    {Par.} & {True} & {MLE} & {$N$ \,} & $\parder{\hat{f}}{v_i}$ & $\parder{f^+}{N_i}$
    \\
    \midrule
    $v_1$   & 0.25   & 0.50  & 12 & 16.00 & -2.74 \\
    $v_2$   & 0.40   & 0.42  & 36 & 2.93 & -0.02 \\
    $v_3$   & 0.45   & 0.63  & 30 & 0.00 & 0.00 \\
    $v_4$   & 0.50   & 0.53  & 60 & 22.96 & -0.07 \\
    $v_5$   & 0.35   & 0.41  & 22  & 8.59 & -0.16 \\
    \bottomrule
\end{tabularx}
}
        
        \caption{MLEs and derivatives.}
        \label{tab:motivating}
    \end{subfigure}
    \begin{subfigure}[b]{0.3\textwidth}
        \centering

        \iftikzcompile
        \resizebox{.85\textwidth}{!}{%
\begin{tikzpicture}[
                nodestyle/.style={draw,circle},
                loop_above/.style={out=110,in=70,looseness=3,->},
                loop_right/.style={out=30,in=-10,looseness=3,->},
                ]
            \node [nodestyle,initial,initial text=] (s11) at (0,0) {$s_{1}$};
            \node [nodestyle] (s12) [on grid, right=1.6cm of s11] {$s_{2}$};
            
            \node [nodestyle] (s21) [on grid, below=1.5cm of s11] {$s_{3}$};
            \node [nodestyle] (s22) [on grid, below=1.5cm of s12] {$s_{4}$};

            \node (s31) [on grid, below=0.9cm of s21] {$\vdots$};
            \node (s32) [on grid, below=0.9cm of s22] {$\vdots$};

            \node (s13) [on grid, right=0.9cm of s12] {$\hdots$};
            \node (s23) [on grid, right=0.9cm of s22] {$\hdots$};

            \draw[->] (s11) -- node [above] {\scriptsize $0.50$} (s12);
            \draw[->] (s12) -- node [right] {\scriptsize $0.50$} (s22);
            \draw(s11) edge[loop_above] node [above] {\scriptsize $0.50$} (s11);
            \draw(s12) edge[loop_above] node [above] {\scriptsize $0.50$} (s12);

            \draw[->] (s21) -- node [above] {\scriptsize $0.58$} (s22);
            \draw[->] (s22) -- node [right] {\scriptsize $0.50$} ($ (s32) + (0,5pt) $);
            \draw(s21) edge[loop_above] node [above] {\scriptsize $0.42$} (s21);
            \draw(s22) edge[loop_right] node [above, xshift=2pt, yshift=3pt] {\scriptsize $0.50$} (s22);
\end{tikzpicture}
}
        \fi
        
        \caption{Portion of the MC.}
        \label{fig:motivating:mc}
    \end{subfigure}
    
    \caption{Grid world environment (a).
    The vehicle (\includegraphics[scale=0.02]{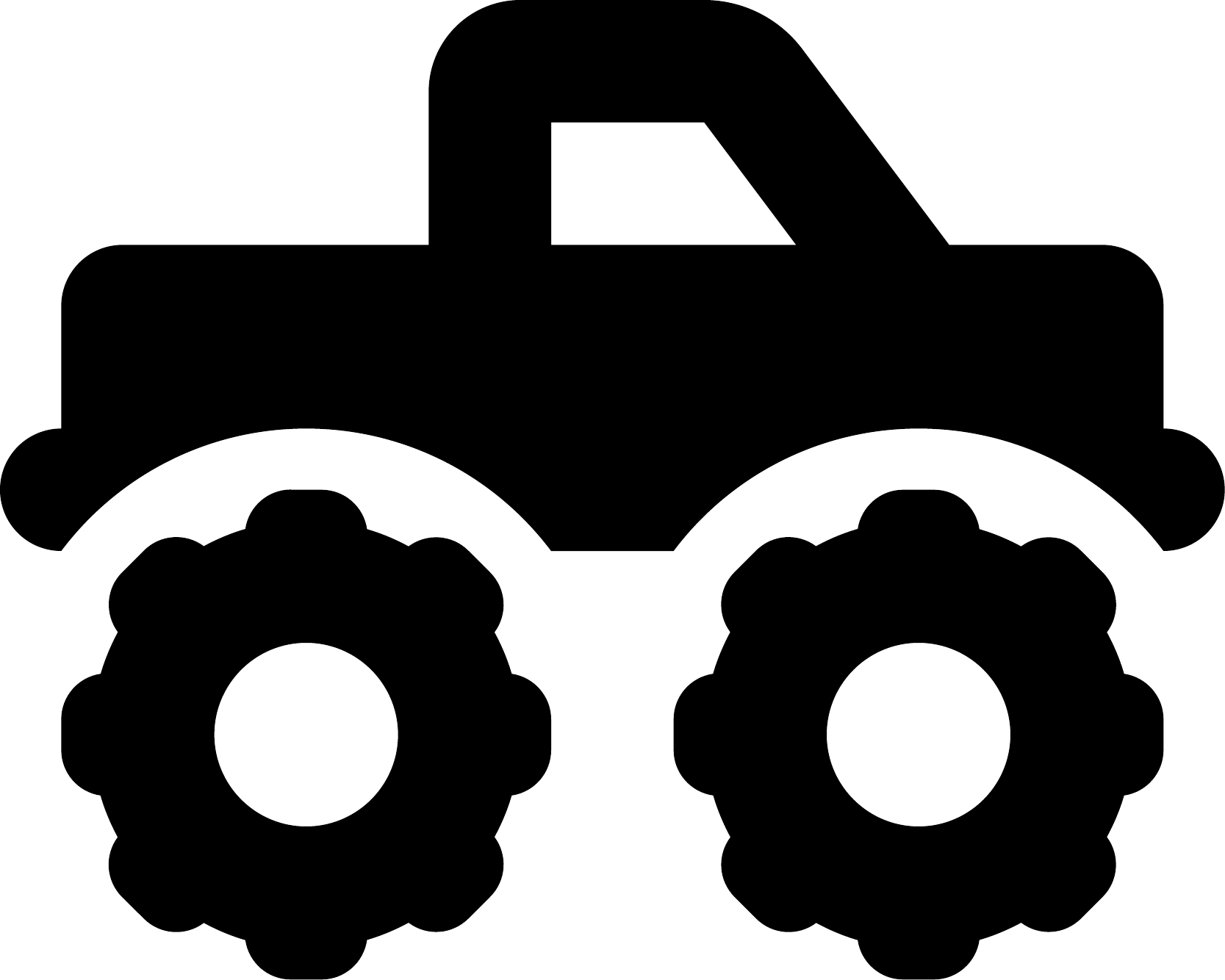}) must deliver the package (\includegraphics[scale=0.02]{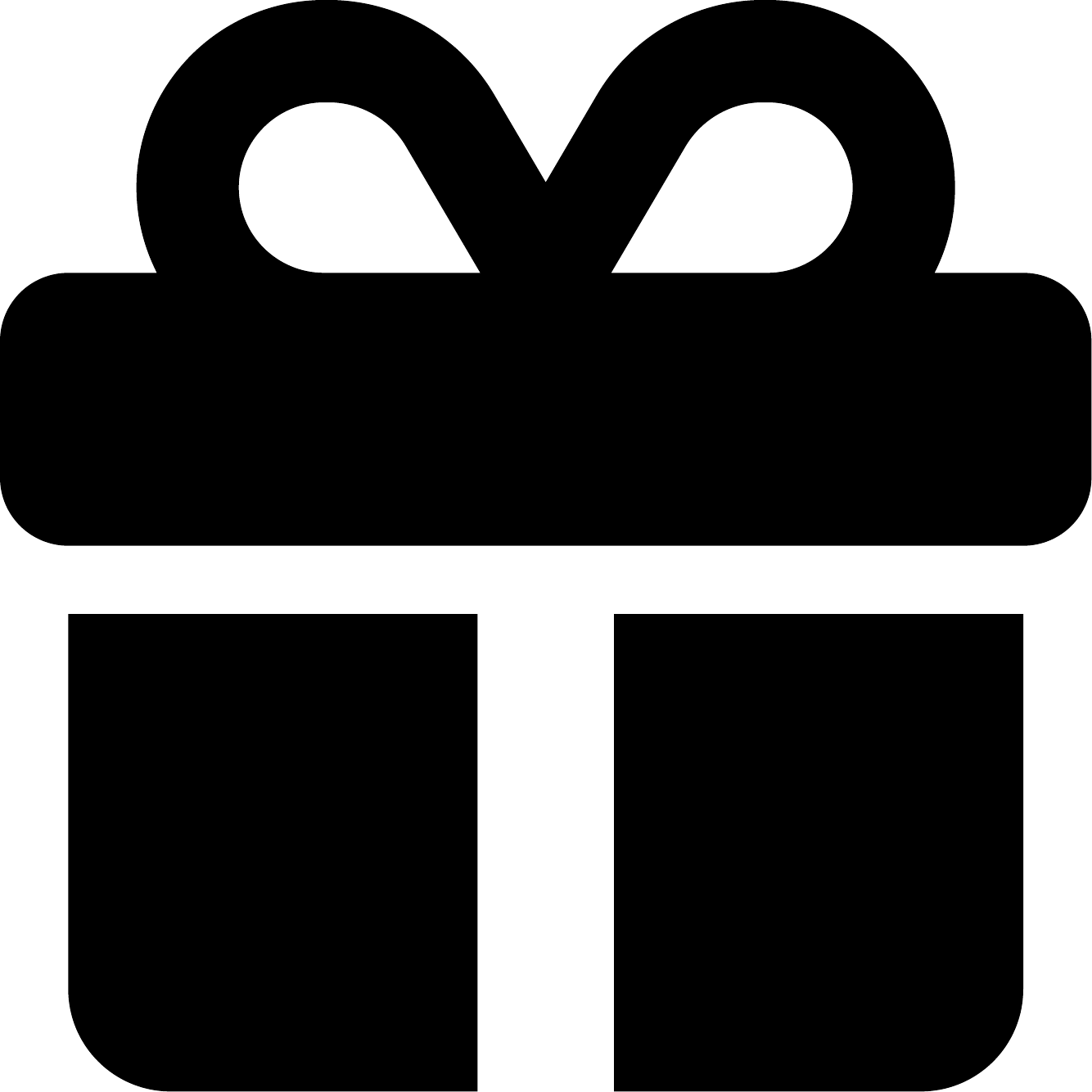}) to the warehouse (\includegraphics[scale=0.02]{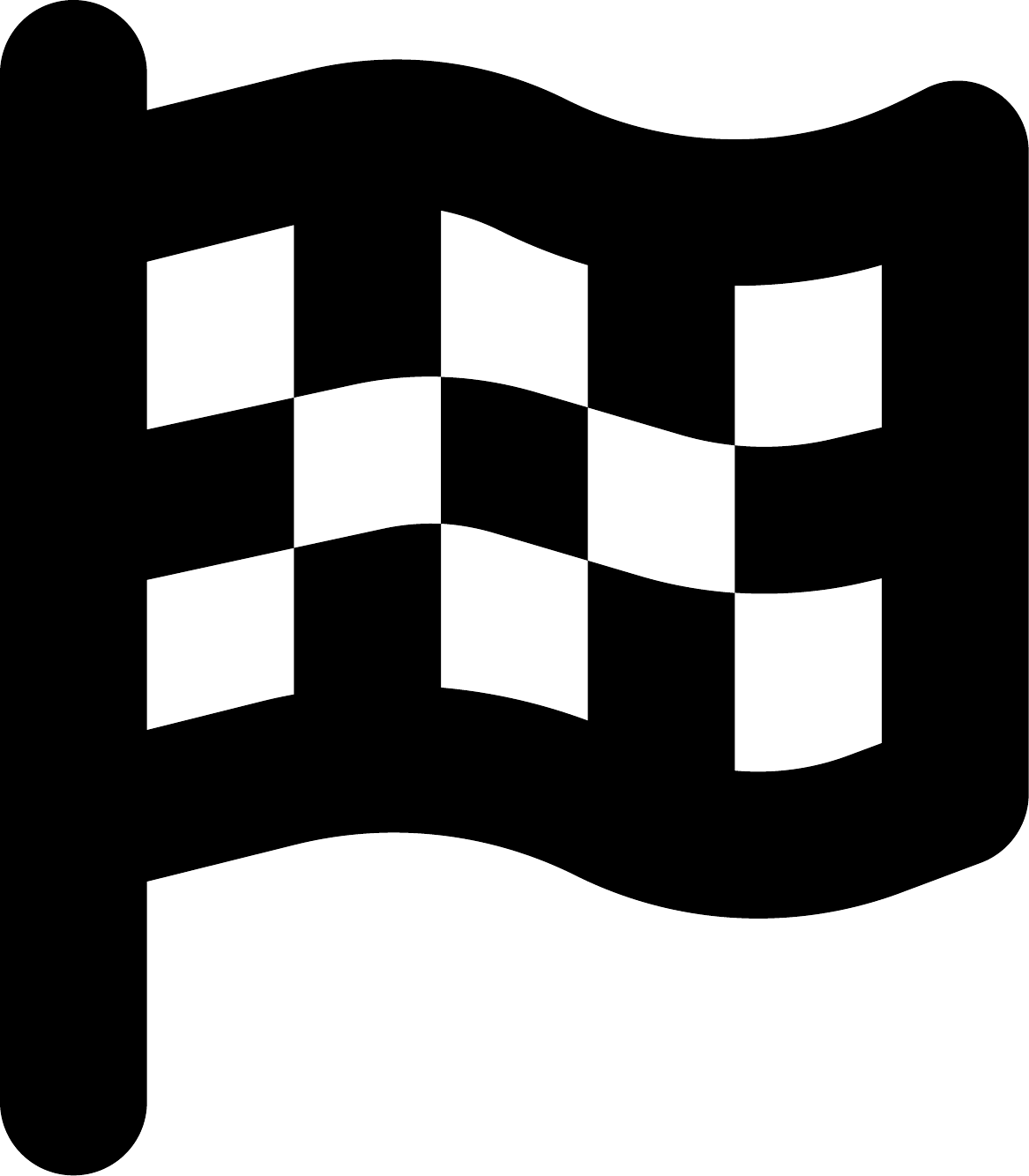}).
    We obtain the MLEs in (b), leading to the MC in (c).}
    \label{fig:motivating}
\end{figure*}

\section{Overview}
\label{sec:motivation}

We expand the example from \cref{sec:introduction} to illustrate our approach more concretely.
The environment, shown in \cref{fig:motivating:terrain}, is partitioned into five regions of the same terrain type.
The vehicle can move in the four cardinal directions.
Recall that the slipping probabilities are the same for all states with the same terrain.
The vehicle follows a dedicated route to collect and deliver a package to a warehouse.
Our goal is to estimate the expected number of steps $f^\star$ to complete the mission. 

\paragraph{Estimating probabilities.}
Classically, we would derive maximum likelihood estimates (MLEs) of the probabilities by sampling.
Consider that, using $N$ samples per slipping probability, we obtained the rough MLEs shown in~\cref{tab:motivating} and thus the MC in~\cref{fig:motivating:mc}.
Verifying the \gls{MC} shows that the expected travel time (called the solution) under these estimates is $\hat{f} = 25.51$ steps, which is far from the travel time of $f^\star = 21.62$ steps under the true slipping probabilities.
We want to close this \emph{verification-to-real gap} by taking more samples for one of the terrain types.
For which of the five terrain types should we obtain more samples? 

\paragraph{Parametric model.}
We can model the grid world as a \gls{pMC}, i.e., an MC with symbolic probabilities.
The solution function for this \gls{pMC} is the travel time $\hat{f}$, being a function of these symbolic probabilities.
We sketch four states of this \gls{pMC} in \cref{fig:motivating:pmc}.
The most relevant parameter is then naturally defined as the parameter with the \emph{largest partial derivative of the solution function}.
As shown in \cref{tab:motivating}, parameter $v_4$ has the highest partial derivative of $\parder{\hat{f}}{v_4} = 22.96$, while the derivative of $v_3$ is zero as no states related to this parameter are ever visited.

\begin{figure}[t!]
\centering
\begin{minipage}{0.35\textwidth}
  \centering
  \iftikzcompile
    \resizebox{.85\linewidth}{!}{%
    \begin{tikzpicture}[
                    nodestyle/.style={draw,circle},
                    loop_above/.style={out=110,in=70,looseness=3,->},
                    loop_right/.style={out=30,in=-10,looseness=3,->},
                    ]
                \node [nodestyle,initial,initial text=] (s11) at (0,0) {$s_{1}$};
                \node [nodestyle] (s12) [on grid, right=1.8cm of s11] {$s_{2}$};
                
                \node [nodestyle] (s21) [on grid, below=1.6cm of s11] {$s_{3}$};
                \node [nodestyle] (s22) [on grid, below=1.6cm of s12] {$s_{4}$};
    
                \node (s31) [on grid, below=0.9cm of s21] {$\vdots$};
                \node (s32) [on grid, below=0.9cm of s22] {$\vdots$};
    
                \node (s13) [on grid, right=0.9cm of s12] {$\hdots$};
                \node (s23) [on grid, right=0.9cm of s22] {$\hdots$};
    
                \draw[->] (s11) -- node [above] {\scriptsize $1-v_1$} (s12);
                \draw[->] (s12) -- node [right] {\scriptsize $1-v_1$} (s22);
                \draw(s11) edge[loop_above] node [above] {\scriptsize $v_1$} (s11);
                \draw(s12) edge[loop_above] node [above] {\scriptsize $v_1$} (s12);
    
                \draw[->] (s21) -- node [above] {\scriptsize $1-v_2$} (s22);
                \draw[->] (s22) -- node [right] {\scriptsize $1-v_1$} ($ (s32) + (0,5pt) $);
                \draw(s21) edge[loop_above] node [above] {\scriptsize $v_2$} (s21);
                \draw(s22) edge[loop_right] node [above, xshift=2pt, yshift=3pt] {\scriptsize $v_1$} (s22);
    \end{tikzpicture}
    }
    \fi
    
    \caption{Parametric MC.}
    \label{fig:motivating:pmc}
\end{minipage}%
\begin{minipage}{0.5\textwidth}
  \centering
    \iftikzcompile
    \resizebox{\linewidth}{!}{%
    \begin{tikzpicture}[
                    nodestyle/.style={draw,circle},
                    loop_above/.style={out=110,in=70,looseness=3,->},
                    loop_right/.style={out=30,in=-10,looseness=3,->},
                    ]
                \node [nodestyle,initial,initial text=] (s11) at (0,0) {$s_{1}$};
                \node [nodestyle] (s12) [on grid, right=3.0cm of s11] {$s_{2}$};
                
                \node [nodestyle] (s21) [on grid, below=1.6cm of s11] {$s_{3}$};
                \node [nodestyle] (s22) [on grid, below=1.6cm of s12] {$s_{4}$};
    
                \node (s31) [on grid, below=0.9cm of s21] {$\vdots$};
                \node (s32) [on grid, below=0.9cm of s22] {$\vdots$};
    
                \node (s13) [on grid, right=0.9cm of s12] {$\hdots$};
                \node (s23) [on grid, right=0.9cm of s22] {$\hdots$};
    
                \draw[->] (s11) -- node [below, align=left] {\scriptsize $[1 - \munderbar{g}(N_1),$ \\ \scriptsize $\quad 1 - \bar{g}(N_1)]$} (s12);
                \draw[->] (s12) -- node [right, pos=0.3] {\scriptsize $[1 - \munderbar{g}(N_1), 1 - \bar{g}(N_1)]$} (s22);
                \draw(s11) edge[loop_above] node [above] {\scriptsize $[\munderbar{g}(N_1), \bar{g}(N_1)]$} (s11);
                \draw(s12) edge[loop_above] node [above] {\scriptsize $[\munderbar{g}(N_1), \bar{g}(N_1)]$} (s12);
    
                \draw[->] (s21) -- node [below, align=left] {\scriptsize $[1 - \munderbar{g}(N_2),$ \\ \scriptsize $\quad 1 - \bar{g}(N_2)]$} (s22);
                \draw[->] (s22) -- node [right] {\scriptsize $[1 - \munderbar{g}(N_1), 1 - \bar{g}(N_1)]$} ($ (s32) + (0,5pt) $);
                \draw(s21) edge[loop_above] node [above] {\scriptsize $[\munderbar{g}(N_2), \bar{g}(N_2)]$} (s21);
                \draw(s22) edge[loop_right] node [above, xshift=20pt, yshift=3pt] {\scriptsize $[\munderbar{g}(N_1), \bar{g}(N_1)]$} (s22);
    \end{tikzpicture}
    }
    \fi
    
    \caption{Parametric robust MC.}
    \label{fig:motivating:prmc}
\end{minipage}
\end{figure}

\paragraph{Parametric robust model.}
The approach above does not account for the uncertainty in each MLE.
Terrain type $v_4$ has the highest derivative but also the largest sample size, so sampling $v_4$ once more has likely less impact than for, e.g., $v_1$.
So, is $v_4$ actually the best choice to obtain additional samples for?
The \gls{prMC} that allows us to answer this question is shown in \cref{fig:motivating:prmc}, where we use (parametric) intervals as uncertainty sets.
The parameters are the sample sizes $N_1, \ldots, N_5$ for all terrain types (contrary to the \gls{pMC}, where parameters represent slipping probabilities).
Now, if we obtain one additional sample for a particular terrain type, how can we expect the uncertainty sets to change?

\paragraph{Derivatives for prMCs.}
We use the \gls{prMC} to compute an upper bound $f^+$ on the true solution $f^\star$.
Obtaining one more sample for terrain type $v_i$ (i.e., increasing $N_i$ by one) shrinks the interval $[\munderbar{g}(N_i), \bar{g}(N_i)]$ on expectation, which in turn decreases our upper bound $f^+$.
Here, $\munderbar{g}$ and $\bar{g}$ are functions mapping sample sizes to interval bounds.
The partial derivatives $\parder{f^+}{N_i}$ for the \gls{prMC} are also shown in \cref{tab:motivating} and give a very different outcome than the derivatives for the \gls{pMC}.
In fact, sampling $v_1$ yields the biggest decrease in the upper bound $f^+$, so we ultimately decide to sample for terrain type $v_1$ instead of $v_4$.

\paragraph{Efficient differentiation.}
We remark that we do not need to know all derivatives explicitly to determine where to obtain samples.
Instead, it suffices to know \emph{which parameter has the highest (or lowest) derivative}.
In the rest of the paper, we develop efficient methods for computing either all or only the $k \in \N$ highest partial derivatives of the solution functions for \glspl{pMC} and \glspl{prMC}.

\paragraph{Supported extensions.}
Our approaches are applicable to general \glspl{pMC} and \glspl{prMC} whose parameters can be shared between distributions (and thus capture dependencies, being a common advantage of parametric models in general~\cite{Junges2019ParameterModels}).
Besides parameters in transition probabilities, we can handle parametric initial states, rewards, and policies.
We could, e.g., use parameters to model the policy of a surveillance drone in our example and compute derivatives for these parameters.
\section{Formal Problem Statement}
\label{sec:problem}

Let $\Param = \{\param_1,\ldots,\param_\ell\}$, $v_i \in \Real$ be a finite and ordered set of parameters.
A parameter instantiation is a function $u \colon \Param \to \Real$ that maps a parameter to a real valuation.
The vector function $\U(v_1, \ldots, v_\ell) = [u(\param_1), \ldots, u(\param_\ell)]^\top \in \Real^\ell$ denotes an ordered instantiation of all parameters in $\Param$ through $u$. 
The set of polynomials over the parameters $V$ is $\Q[\Param]$.
A polynomial $f$ can be interpreted as a function $f \colon \Real^\ell \to \Real$ where $f(\U)$ is obtained by substituting each occurrence of $\param$ by $u(\param)$.
We denote these substitutions with $f[\U]$.

For any set $X$, let $\pfun{X} = \{ f \mid f \colon X \to \Q[\Param] \}$ be the set of functions that map from $X$ to the polynomials over the parameters $\Param$.
We denote by $\pdistr{X} \subset \pfun{X}$ the set of \emph{parametric probability distributions} over $X$, \ie, the functions $f \colon X \to \Q[\Param]$ such that $f(x)[\U] \in [0,1]$ and $\sum_{x \in X} f(x)[\U] = 1$ for all parameter instantiations $\U$. 

\subsubsection{Parametric Markov Chains}
We define a \gls{pMC} as follows:
\begin{definition}[pMC]
    \label{def:pMC}
    A \gls{pMC} $\pmc$ is a tuple $\PMC$, where $\States$ is a finite set of states, $\sI \in \distr{\States}$ a distribution over initial states, $\Param$ a finite set of parameters, and $\transfunc \colon \States \to \pdistr{\States}$ a parametric transition~function.
\end{definition}
Applying an instantiation $\U$ to a \gls{pMC} yields an \gls{MC} $\pmc[\U]$ by replacing each transition probability $f \in \Q[\Param]$ by $f[\U]$.
We consider expected reward measures based on a state reward function $R \colon \States \to \Real$.
Each parameter instantiation for a \gls{pMC} yields an \gls{MC} for which we can compute the solution for the expected reward measure~\cite{DBLP:books/daglib/BaierKatoen2008}.
We call the function that maps instantiations to a solution the \emph{solution function}.
The solution function is smooth over the set of graph-preserving instantiations~\cite{DBLP:journals/jcss/JungesK0W21}.
Concretely, the solution function $\ExpR$ for the expected cumulative reward under instantiation $\U$ is written as follows:
\begin{equation}
    \label{eq:solution_pmc}
    \ExpR (\U) = \sum_{s \in \States} \Big( \sI(s) \sum_{\omega \in \Omega(s)} \text{rew}(\omega) \cdot \Pr(\omega, \U) \Big),
\end{equation}
where $\Omega(s)$ is the set of paths starting in $s \in \States$, $\text{rew}(\omega) = R(s_0) + R(s_1) + \cdots$ is the cumulative reward over $\omega = s_0 s_1 \cdots$, and $\Pr(\omega, \U)$ is the probability for a path $\omega \in \Omega(s)$.
If a terminal (sink) state is reached from state $s \in \States$ with probability one, the infinite sum over $\omega \in \Omega(s)$ in \cref{eq:solution_pmc} exist~\cite{DBLP:books/wi/Puterman94}.

\subsubsection{Parametric Robust Markov Chains}
The convex polytope $T_{A,b} \subseteq \Real^n$ defined by matrix $A \in \Real^{m \times n}$ and vector $b \in \Real^m$ is the set $T_{A,b} = \{ p \in \Real^n \mid Ap \leq b \}$.
We denote by $\F_n$ the set of all convex polytopes of dimension $n$, i.e.,
\begin{equation}
\begin{split}
    \label{eq:polytope_set_concrete}
    \F_n = \{ 
        T_{A,b} \mid A \in \Real^{m\times n},
        \, b \in \Real^m, \, m \in \N
    \}.%
\end{split}%
\end{equation}%
A \glsfirst{rMC}~\cite{DBLP:journals/ior/WiesemannKS14,DBLP:conf/tacas/SenVA06} is a tuple $\RMC$, where $\States$ and $\sI$ are defined as for \glspl{pMC} and the uncertain transition function $\transfuncImdp \colon \States \to \F_{|\States|}$ maps states to convex polytopes $T \in \F_{|\States|}$.
Intuitively, an \gls{rMC} is an \gls{MC} with possibly infinite \emph{sets of probability distributions}.
To obtain robust bounds on the verification result for any of these \glspl{MC}, an \emph{adversary} nondeterministically chooses a precise transition function by fixing a probability distribution $\hat{P}(s) \in \transfuncImdp(s)$ for each $s \in \States$.

We extend \glspl{rMC} with polytopes whose halfspaces are defined by polynomials $\Q[\Param]$ over $\Param$.
To this end, let $\F_n[\Param]$ be the set of all such \emph{parametric polytopes}:
\begin{equation}
\begin{split}
    \label{eq:polytope_set_parametric}
    \F_n[\Param] = \{ 
        T_{A,b} \mid A \in \Q[\Param]^{m\times n},
        \, b \in \Q[\Param]^m, \, m \in \N
    \}.
\end{split}
\end{equation}
An element $T \in \F_n[\Param]$ can be interpreted as a function $T \colon \Real^\ell \to 2^{(\Real^n)}$ that maps an instantiation $\U$ to a (possibly empty) convex polytopic subset of $\Real^n$.
The set $T[\U]$ is obtained by substituting each $\param_i$ in $T$ by $u(\param_i)$ for all $i = 1,\ldots,\ell$.

\begin{example}
    The uncertainty set for state $s_{1}$ of the \gls{prMC} in \cref{fig:motivating:prmc} is the parametric polytope $T \in \F_{2}[\Param]$ with singleton parameter set $\Param = \{ N_1\}$, such that
    \begin{equation*}
    \begin{split}
        T = \big\{ 
            [p_{1,1}, p_{1,2}]^\top \in \Real^2 
            \,\,\big\vert\,\, 
            & \munderbar{g}_1(N_1) \leq p_{1,1} \leq \bar{g}_1(N_1), 
            \enskip
            \\
            & 1 - \bar{g}_1(N_1) \leq p_{1,2} \leq 1 - \munderbar{g}_1(N_1),
            \enskip
            \, p_{1,2} + p_{1,2} = 1
        \big\}.
    \end{split}
    \end{equation*}
\end{example}
We use parametric convex polytopes to define \glspl{prMC}:
\begin{definition}[prMC]
    \label{def:prMC}
    A \gls{prMC} $\prmc$ is a tuple $\PRMC$, where $\States$, $\sI$, and $\Param$ are defined as for \glspl{pMC} (\cref{def:pMC}), and where
    $\transfuncImdp \colon \States \to \F_{|\States|}[\Param]$ is a parametric and uncertain transition function that maps states to parametric convex polytopes.
\end{definition}
Applying an instantiation $\U$ to a prMC yields an \gls{rMC} $\prmc[\U]$ by replacing each parametric polytope $T \in \F_{|\States|}[\Param]$ by $T[\U]$, \ie, a polytope defined by a concrete matrix $A \in \Real^{m \times n}$ and vector $b \in \Real^{m}$.
Without loss of generality, we consider adversaries minimizing the expected cumulative reward until reaching a set of terminal states $\terminalStates \subseteq \States$.
This minimum expected cumulative reward $\ExpRR (\U)$, called the \emph{robust solution} on the instantiated prMC $\prmc[\U]$, is defined as
\begin{equation}
    \label{eq:solution_prmc}
    \ExpRR (\U) = \sum_{s \in \States} \Big( \sI(s) \cdot \min_{P \in \transfuncImdp[\U]} \sum_{\omega \in \Omega(s)} \text{rew}(\omega) \cdot \Pr(\omega, \U, P) \Big).
\end{equation}
We refer to the function $\ExpRR \colon \Real^\ell \to \Real$ as the \emph{robust solution function}.

\paragraph{Assumptions on \glspl{pMC} and \glspl{prMC}.}
For both \glspl{pMC} and \glspl{prMC}, we assume that transitions cannot vanish under any instantiation (graph-preservation).
That is, for every $s,s' \in \States$, we have that $P(s)[\U](s')$ (for \glspl{pMC}) and $\transfuncImdp(s)[\U](s')$ (for \glspl{prMC}) are either zero or strictly positive for all instantiations $\U$.

\subsubsection{Problem statement}
Let $f(q_1, \ldots, q_n) \in \Real^m$ be a differentiable multivariate function with $m \in \N$.
We denote the \emph{partial derivative} of $f$ with respect to $q$ by $\frac{\partial x}{\partial q} \in \Real^m$.
The \emph{gradient} of $f$ combines all partial derivatives in a single vector as $\nabla_q f = [\parder{f}{q_1}, \ldots, \parder{f}{q_n}] \in \Real^{m \times n}$.
We only use gradients $\nabla_{\U} f$ with respect to the parameter instantiation $\U$, so we simply write $\nabla f$ in the remainder.

The gradient of the robust solution function evaluated at the instantiation $\U$ is 
$\grad{\U}{\ExpRR} [\U] = 
\begin{bmatrix}
    \big( \frac{\partial \ExpRR}{\partial u(\param_1)} \big)[\U],
    & \ldots, &
    \big( \frac{\partial \ExpRR}{\partial u(\param_\ell)} \big)[\U]
\end{bmatrix}$.
We solve the following problem.
\begin{mdframed}
\begin{problem}
    \label{prob:formal}
    {\itshape
    Given a prMC $\prmc$ and a parameter instantiation $\U$, compute the gradient $\grad{\U}{\ExpRR}[\U]$ of the robust solution function evaluated at $\U$.
    }
\end{problem}
\end{mdframed}
Solving Problem~\ref{prob:formal} is linear in the number of parameters, which may lead to significant overhead if the number of parameters is large. Typically, it suffices to only obtain the parameters with the highest derivatives:
\begin{mdframed}
\begin{problem}
    \label{prob:formal2}
    {\itshape
    Given a prMC $\prmc$, an instantiation $\U$, and a $k \leq |\Param|$, compute a subset $\Param^\star$ of $k$ parameters for which the partial derivatives are maximal.
    }
\end{problem}
\end{mdframed}
For both problems, we present polynomial-time algorithms for \glspl{pMC} (\cref{sec:differentiating_pMCs}) and \glspl{prMC} (\cref{sec:differentiating_prMCs}). 
\cref{sec:experiments} defines problem variations that we study empirically.
\section{Differentiating Solution Functions for pMCs}
\label{sec:differentiating_pMCs}

We can compute the solution of an \gls{MC} $\pmc[\U]$ with instantiation $\U$ based on a system of $|\States|$ linear equations; here for an expected reward measure~\cite{DBLP:books/daglib/BaierKatoen2008}.
Let $x = [x_{s_1}, \ldots, x_{s_{|\States|}}]^\top$ and $r = [r_{s_1}, \ldots, r_{s_{|\States|}}]^\top$ be variables for the expected cumulative reward and the instantaneous reward in each state $s \in \States$, respectively.
Then, for a set of terminal (\emph{sink}) states $\terminalStates \subset \States$, we obtain the equation system
\begin{subequations}
\begin{alignat}{2}
    & x_s = 0, \quad &&\forall s \in \terminalStates
    \\
    & x_s = r_s + P(s)[\U] x, \quad &&\forall s \in \States \backslash \terminalStates.
\end{alignat}%
\label{eq:pMC_eqsys}%
\end{subequations}%
Let us set $P(s)[\U] = 0$ for all $s \in \terminalStates$ and define  the matrix $P[\U] \in \Real^{|\States| \times |\States|}$ by stacking the rows $P(s)[\U]$ for all $s \in \States$.
Then, \cref{eq:pMC_eqsys} is written in matrix form as $(I_{|\States|} - P[\U]) x = r$.
The equation system in \cref{eq:pMC_eqsys} can be efficiently solved by, \eg, Gaussian elimination or more advanced iterative equation solvers.

\subsection{Computing derivatives explicitly}
We differentiate the equation system in \cref{eq:pMC_eqsys} with respect to an instantiation $u(\param_i)$ for parameter $\param_i \in \Param$,
similar to, e.g.,~\cite{DBLP:conf/vmcai/HeckSJMK22}.
For all $s \in \terminalStates$, the derivative $\parder{x_s}{u(\param_i)}$ is trivially zero.
For all $s \in \States \setminus \terminalStates$, we obtain via the product rule that
\begin{equation}
    \label{eq:pMC:diff_eq_sys1}
    \parder{x_s}{u(\param_i)}
    = \parder{P(s) x}{u(\param_i)} [\U] 
    = (x^\star)^\top \parder{P(s)^\top}{u(\param_i)} [\U]
      + P(s)[\U] \parder{x}{u(\param_i)},
\end{equation}
where $x^\star \in \Real^{|\States|}$ is the solution to \cref{eq:pMC_eqsys}.
In matrix form for all $s \in \States$, this yields
\begin{equation}
    \label{eq:pMC:diff_eq_sys2}
    \left( I_{|\States|} - P[\U] \right) \parder{x}{u(\param_i)}
    = \parder{P x^\star}{u(\param_i)} [\U].
\end{equation}
The solution defined in \cref{eq:solution_pmc} is computed as $\ExpR[\U] = \sI^\top x^\star$.
Thus, the partial derivative of the solution function with respect to $u(\param_i)$ in closed form is
\begin{equation}
    \label{eq:pMC:diff_eq_sys3}
    \left( \parder{\ExpR}{u(\param_i)} \right)[\U] 
    = \sI^\top \parder{x}{u(\param_i)}
    = \sI^\top \left( I_{|\States|} - P[\U] \right)^{-1} \parder{P x^\star}{u(\param_i)} [\U].
\end{equation}
\paragraph{Algorithm for \cref{prob:formal}.}
Let us provide an algorithm to solve \cref{prob:formal} for \glspl{pMC}.
\cref{eq:pMC:diff_eq_sys3} provides a closed-form expression for the partial derivative of the solution function, which is a function of the vector $x^\star$ in \cref{eq:pMC_eqsys}.
However, due to the inversion of $(I_{|\States|} - P[\U])$, it is generally more efficient to solve the system of equations in \cref{eq:pMC:diff_eq_sys2}.
Doing so, the partial derivative of the solution with respect to $u(\param_i)$ is obtained by: (1) solving \cref{eq:pMC_eqsys} with $\U$ to obtain $x^\star \in \Real^{|\States|}$, and (2) solving the equation system in \cref{eq:pMC:diff_eq_sys2} with $|\States|$ unknowns for this vector $x^\star$.
We repeat step 2 for all of the $|\Param|$ parameters.
Thus, we can solve \cref{prob:formal} by solving $|\Param|+1$ linear equation systems with $|\States|$ unknowns each.

\subsection{Computing $k$-highest derivatives}
\label{subsec:pmc_faster}
To solve \cref{prob:formal2} for \glspl{pMC}, we present a method to compute only the $k \leq \ell = |\Param|$ parameters with the highest (or lowest) partial derivative without computing all derivatives explicitly.
Without loss of generality, we focus on the highest derivative.
We can determine these parameters by solving a combinatorial optimization problem with binary variables $z_i \in \{0,1\}$ for $i=1,\ldots,\ell$.
Our goal is to formulate this optimization problem such that an optimal value of $z^\star_i = 1$ implies that parameter $\param_i \in \Param$ belongs to the set of $k$ highest derivatives.
Concretely, we formulate the following \emph{mixed integer linear problem} (MILP)~\cite{wolsey2020integer}:
\begin{subequations}%
\label{eq:pMC_importance_MIP}%
\begin{align}%
    \maximize_{y \in \Real^{|\States|}, \, z \in \{ 0, 1 \}^{\ell}} \,\, & \sI^\top y
    \label{eq:pMC_importance_MIP:obj}
    \\
    \text{subject to} \,\, & 
    \left( I_{|\States|} - P[\U] \right) y = \sum_{i=1}^{\ell} z_i \parder{P x^\star}{u(\param_i)}[\U]
    \label{eq:pMC_importance_MIP:cns1}
    \\
    & z_1 + \cdots + z_{\ell} = k.
    \label{eq:pMC_importance_MIP:cns2}%
\end{align}%
\end{subequations}%
Constraint (\ref{eq:pMC_importance_MIP:cns2}) ensures that any feasible solution to \cref{eq:pMC_importance_MIP} has exactly $k$ nonzero entries.
Since matrix $(I_{|\States|} - P[\U])$ is invertible by construction~(see, e.g., \cite{DBLP:books/wi/Puterman94}), \cref{eq:pMC_importance_MIP} has a unique solution in $y$ for each choice of $z \in \{0, 1\}^{\ell}$.
Thus, the objective value $\sI^\top y$ is the sum of the derivatives for the parameters $\param_i \in \Param$ for which $z_i = 1$.
Since we maximize this objective, an optimal solution $y^\star, z^\star$ to \cref{eq:pMC_importance_MIP} is guaranteed to correspond to the $k$ parameters that maximize the derivative of the solution in \cref{eq:pMC:diff_eq_sys3}.
We state this correctness claim for the MILP:
\begin{proposition}%
    \label{prop:pMC}%
    Let $y^\star$, $z^\star$ be an optimal solution to \cref{eq:pMC_importance_MIP}. 
    Then, the set $\Param^\star = \{ \param_i \in \Param 
    \,\,\vert\,\, 
    z_i^\star = 1 \}$ is a subset of $k \leq \ell$ parameters with maximal derivatives.
\end{proposition}
The set $\Param^\star$ may not be unique.
However, to solve \cref{prob:formal2}, it suffices to obtain \emph{a set} of $k$ parameters for which the partial derivatives are maximal.
Therefore, the set $\Param^\star$ provides a solution to \cref{prob:formal2}.
We remark that, to solve \cref{prob:formal2} for the $k$ lowest derivatives, we change the objective in \cref{eq:pMC_importance_MIP:obj} to $\minimize \sI^ \top y$.

\paragraph{Linear relaxation.}
The MILP in \cref{eq:pMC_importance_MIP} is computationally intractable for high values of $\ell$ and $k$.
Instead, we compute the set $v^\star$ via a \emph{linear relaxation} of the MILP.
Specifically, we relax the binary variables $z \in \{0,1\}^\ell$ to continuous variables $z \in [0,1]^\ell$.
As such, we obtain the following LP relaxation of \cref{eq:pMC_importance_MIP}:
\begin{subequations}
\label{eq:pMC_importance_LP}
\begin{align}
    \maximize_{y \in \Real^{|\States|}, \, z \in \Real^{\ell}} \,\, & \sI^\top y
    \label{eq:pMC_importance_LP:obj}
    \\
    \text{subject to} \,\, & 
    \left( I_{|\States|} - P[\U] \right) y = \sum_{i=1}^{\ell} z_i \parder{P x^\star}{u(\param_i)}[\U]
    \label{eq:pMC_importance_LP:cns1}
    \\
    & 0 \leq z_i \leq 1, \quad \forall i = 1,\ldots,\ell
    \label{eq:pMC_importance_LP:cns2}
    \\
    & z_1 + \cdots + z_{\ell} = k.
    \label{eq:pMC_importance_LP:cns3}
\end{align}
\end{subequations}
Denote by $y^+, z^+$ the solution of the LP relaxation in \cref{eq:pMC_importance_LP}.
For details on such linear relaxations of integer problems, we refer to~\cite{DBLP:books/daglib/p/HoffmanK10,Matousek2007}.
In our case, every optimal solution $y^+, z^+$ to the LP relaxation with only binary values $z_i^+ \in \{0, 1\}$ is also optimal for the MILP, resulting in the following theorem.
\begin{theorem}
    \label{thm:LP_relaxation}
    The LP relaxation in \cref{eq:pMC_importance_LP} has an optimal solution $y^+$, $z^+$ with $z^+ \in \{0,1\}^\ell$ (i.e., every optimal variable $z_i^+$ is binary), and every such a solution is also an optimal solution of the MILP in \cref{eq:pMC_importance_MIP}.
\end{theorem}

\begin{proof}
From invertibility of $\left( I_{|\States|} - P[\U] \right)$, we know that \cref{eq:pMC_importance_MIP} is equivalent to 
\begin{subequations}
\label{eq:pMC_importance_LP_reformulation}
\begin{align}
    \maximize_{ z \in \{ 0, 1 \}^{\ell}} \,\, &  \sum_{i=1}^{\ell} z_i \left( \sI^\top \left( I_{|\States|} - P[\U] \right)^{-1}\parder{P x^\star}{u(\param_i)}[\U]\right)
    \\
    \text{subject to} \,\, 
    & z_1 + \cdots + z_{\ell} = k.
\end{align}%
\end{subequations}%
The linear relaxation of \cref{eq:pMC_importance_LP_reformulation} is an LP whose feasible region has integer vertices (see, e.g., \cite{hoffman2010integral}). 
Therefore, both \cref{eq:pMC_importance_LP_reformulation} and its relaxation \cref{eq:pMC_importance_LP} have an integer optimal solution $z^+$, which constructs $z^\star$ in \cref{eq:pMC_importance_MIP}. 
\qed
\end{proof}

The binary solutions $z^+ \in \{0,1\}^\ell$ are the vertices of the feasible set of the LP in \cref{eq:pMC_importance_LP}. 
A simplex-based LP solver can be set to return such a solution.\footnote{Even if a non-vertex solution $y^+,z^+$ is obtained, we can use an arbitrary tie-break rule on $z^+$, which forces each $z^+_i$ binary and preserves the sum in \cref{eq:pMC_importance_LP:cns3}.}

\paragraph{Algorithm for \cref{prob:formal2}.}
We provide an algorithm to solve \cref{prob:formal2} for \glspl{pMC} consisting of two steps.
First, for \gls{pMC} $\pmc$ and parameter instantiation $\U$, we solve the linear equation system in \cref{eq:pMC:diff_eq_sys2} for $x^\star$ to obtain the solution $\ExpR[\U] = \sI^\top x^\star$.
Second, we fix a number of parameters $k \leq \ell$ and solve the LP relaxation in \cref{eq:pMC_importance_LP}.
The set $\Param^\star$ of parameters with maximal derivatives is then obtained as defined in \cref{prop:pMC}.
The parameter set $\Param^\star$ is a solution to \cref{prob:formal2}.
\section{Differentiating Solution Functions for prMCs}
\label{sec:differentiating_prMCs}

We shift focus to \glspl{prMC}.
Recall that solutions $\ExpRR[\U]$ are computed for the worst-case realization of the uncertainty, called the robust solution.
We derive the following equation system, where, 
as for \glspl{pMC}, $x \in \Real^{|\States|}$ represents the expected cumulative reward in each state.
\begin{subequations}
\begin{alignat}{2}
    & x_s = 0, &&\forall s \in \terminalStates
    \label{eq:prMC_eqsys_constr_terminal}
    \\
    & x_s = r_s + \inf_{p \in \transfuncImdp(s)[\U]} \left( p^\top x \right), \quad &&\forall s \in \States \setminus \terminalStates.
    \label{eq:prMC_eqsys_constr}
\end{alignat}%
\label{eq:prMC_eqsys}%
\end{subequations}%
Solving \cref{eq:prMC_eqsys} directly corresponds to solving a system of nonlinear equations due to the inner infimum in \cref{eq:prMC_eqsys_constr}.
The standard approach from robust optimization~\cite{DBLP:books/degruyter/Ben-TalGN09} is to leverage the dual problem for each inner infimum, e.g., as is done in~\cite{DBLP:conf/cav/PuggelliLSS13,DBLP:conf/concur/ChenFRS14}.
For each $s \in \States$, $\transfuncImdp(s)$ is a parametric convex polytope $T_{A,b}$ as defined in \cref{eq:polytope_set_parametric}.
The dimensionality of this polytope depends on the number of successor states, which is typically much lower than the total number of states.
To make the number of successor states explicit, we denote by $\post{s} \subseteq \States$ the successor states of $s \in \States$ and define $T_{A,b} \in \F_{|\post{s}|}[\Param]$ with $A_s \in \Q^{m_s \times |\post{s}|}$ and $b_s[\U] \in \Q^{m_s}$ (recall $m_s$ is the number of halfspaces of the polytope).
Then, the infimum in \cref{eq:prMC_eqsys_constr} for each $s \in \States \setminus \terminalStates$ is
\begin{subequations}
\begin{align}
    \minimize \,\, & p^\top x
    \\
    \text{subject to} \,\, 
    & A_s[\U] p \leq b_s[\U] \\
    & \mathbbm{1}^\top p = 1,
\end{align}%
\label{eq:inner_inf1}%
\end{subequations}%
where $\mathbbm{1}$ denotes a column vector of ones of appropriate size.
Let $x_{\post{s}} = [x_s]_{s \in \post{s}}$ be the vector of decision variables corresponding to the (ordered) successor states in $\post{s}$.
The dual problem of \cref{eq:inner_inf1}, with dual variables $\alpha \in \Real^{m_s}$ and $\beta \in \Real$ (see, e.g.,~~\cite{bazaraa2011linear} for details), is written as follows:
\begin{subequations}
\begin{align}
    \maximize\enskip & {-}b_s[\U]^\top \alpha - \beta
    \\
    \text{subject to} \,\, & A_s[\U]^\top \alpha + x_{\post{s}} + \beta\mathbbm{1} = 0
    \\
    & \alpha \geq 0.
\end{align}%
\label{eq:inf_dual_problem}%
\end{subequations}%
By using this dual problem in \cref{eq:prMC_eqsys_constr}, we obtain the following LP with decision variables $x \in \Real^{|\States|}$, and with $\alpha_{s} \in \Real^{m_{s}}$ and $\beta_{s} \in \Real$ for every $s \in \States$:
\begin{subequations}
\label{eq:prmc_LP}
\begin{alignat}{2}
    \label{eq:prmc_LP_obj}
    \maximize \,\, & \sI^\top x
    \\
    \text{subject to} \,\, 
    & x_s = 0, &&\forall s \in \terminalStates
    \label{eq:prmc_LP_constr1}
    \\
    & 
    x_s = r_s - \left( b_{s}[\U]^\top \alpha_{s} + \beta_{s} \right),
    \enskip &&\forall s \in \States \setminus \terminalStates
    \label{eq:prmc_LP_constr2}
    \\    
    & A_{s}[\U]^\top \alpha_{s} + x_{\post{s}} + \beta_{s}\mathbbm{1} = 0, \quad \alpha_{s} \geq 0, \quad && \forall s \in \States \setminus \terminalStates.
\end{alignat}%
\end{subequations}%
The reformulation of \cref{eq:prMC_eqsys} to \cref{eq:prmc_LP} requires that $\sI \geq 0$, which is trivially satisfied because $\sI$ is a probability distribution.
Denote by $x^\star, \alpha^\star, \beta^\star$ an optimal point of \cref{eq:prmc_LP}.
The $x^\star$ element of this optimum is also an optimal solution of \cref{eq:prMC_eqsys}~\cite{DBLP:books/degruyter/Ben-TalGN09}.
Thus, the robust solution defined in \cref{eq:solution_prmc} is $\ExpRR[\U] = \sI^\top x^\star$.

\subsection{Computing derivatives via pMCs (and when it does not work)}
\label{sec:differentiating_prMCs:intuition}
Toward solving \cref{prob:formal}, we provide some intuition about computing robust solutions for \glspl{prMC}.
The infimum in \cref{eq:prMC_eqsys_constr} finds the \emph{worst-case} point $p^\star$ in each set $\transfuncImdp(s)[\U]$ that minimizes ${(p^\star)}^\top x$.
This minimization is visualized in \cref{fig:uncertainty_set1} for an uncertainty set that captures three probability intervals $\munderbar{p}_i \leq p_i \leq \bar{p}_i, \, i=1,2,3$.
Given the optimization direction $x$ (arrow in \cref{fig:uncertainty_set1}), the point $p^\star$ (red dot) is attained at the vertex where the constraints $\munderbar{p}_1 \leq p_1$ and $\munderbar{p}_2 \leq p_2$ are active.\footnote{An inequality constraint $g x \leq h$ is active under the optimal solution $x^\star$ if $g x^\star = h$~\cite{DBLP:books/cu/BV2014}.}
Thus, we obtain that the point in the polytope that minimizes ${(p^\star)}^\top x$ is $p^\star = [\munderbar{p}_1, \, \munderbar{p}_2, \, 1-\munderbar{p}_1-\munderbar{p}_2]^\top$.
Using this procedure, we can obtain a worst-case point $p^\star_s$ for each state $s \in \States$.
We can use these points to convert the \gls{prMC} into an induced \gls{pMC} with transition function $\transfunc(s) = p^\star_s$ for each state $s \in \States$.

\newcommand\pgfmathsinandcos[3]{%
  \pgfmathsetmacro#1{sin(#3)}%
  \pgfmathsetmacro#2{cos(#3)}%
}

\begin{figure*}[t]
    \centering
    \begin{subfigure}[b]{0.31\textwidth}
        \centering

        \iftikzcompile
        \begin{tikzpicture}[scale=1.5] 

            \pgfmathsetmacro\AngleFuite{150}
            \pgfmathsetmacro\coeffReduc{.8}
            \pgfmathsetmacro\clen{2}
            \pgfmathsinandcos\sint\cost{\AngleFuite}
            
            \begin{scope} [x     = {(\coeffReduc*\cost,-\coeffReduc*\sint)},
                           y     = {(1cm,0cm)}, 
                           z     = {(0cm,1cm)}]
                
                \newcommand\minA{0.15}
                \newcommand\maxA{0.7}
                \newcommand\minB{0.2}
                \newcommand\maxB{0.55}
                \newcommand\minC{0.2}
                \newcommand\maxC{0.75}
        
                \path coordinate (O) at (0,0,0)
                      coordinate (A) at (1,0,0)
                      coordinate (B) at (0,1,0)
                      coordinate (C) at (0,0,1);
                \path coordinate (A_min_1) at (\minA,1-\minA,0)
                      coordinate (A_min_2) at (\minA,0,1-\minA)
                      coordinate (B_min_1) at (1-\minB,\minB,0)
                      coordinate (B_min_2) at (0,\minB,1-\minB)
                      coordinate (C_min_1) at (1-\minC,0,\minC)
                      coordinate (C_min_2) at (0,1-\minC,\minC);
                \path coordinate (A_max_1) at (\maxA,1-\maxA,0)
                      coordinate (A_max_2) at (\maxA,0,1-\maxA)
                      coordinate (B_max_1) at (1-\maxB,\maxB,0)
                      coordinate (B_max_2) at (0,\maxB,1-\maxB)
                      coordinate (C_max_1) at (1-\maxC,0,\maxC)
                      coordinate (C_max_2) at (0,1-\maxC,\maxC);
                
                \draw[gray] (A)--(B)--(C)--cycle;
        
                \draw[fill=gray] (A) circle (1.5pt) node[below, yshift=-0.1cm, xshift=0.2cm] {\small $P_{1}=1$};
                \draw[fill=gray] (B) circle (1.5pt) node[below] {$\small P_{2}=1$};
                \draw[fill=gray] (C) circle (1.5pt) node[right, yshift=0.1cm] {$\small P_{3}=1$};
                
                \draw[thick,-stealth,black] (O) -- (0,0,1);
                \draw[thick,-stealth,black] (O) -- (1,0,0);
                \draw[thick,-stealth,black] (O) -- (0,1,0);
        
                \draw[fill=red] (0.15, 0.2, 0.65) circle (1.5pt) node[red, right] {\small $p^\star$};
                
                \draw[name path=Amin, densely dotted, thick, red] (A_min_1) -- (A_min_2) node[left,yshift=0.1cm] {\scriptsize $\munderbar{p}_{1}$};
                \draw[name path=Bmin, densely dotted, thick, red] (B_min_2) -- (B_min_1) node[below, xshift=0.2cm] {\scriptsize $\munderbar{p}_{2}$};
                \draw[name path=Cmin, densely dotted, gray] (C_min_1) -- (C_min_2) node[right,yshift=0.1cm] {\scriptsize $\munderbar{p}_{3}$};
                
                \draw[name path=Amax, densely dotted, gray] (A_max_1) -- (A_max_2) node[left,yshift=0.1cm] {\scriptsize $\bar{p}_{1}$};
                \draw[name path=Bmax, densely dotted, gray] (B_max_2) -- (B_max_1) node[below, xshift=0.2cm] {\scriptsize $\bar{p}_{2}$};
                \draw[name path=Cmax, densely dotted, gray] (C_max_1) -- (C_max_2) node[right] {\scriptsize $\bar{p}_{3}$};

                \clip (A_min_1) -- (A_min_2) -- (A) -- cycle;
                \clip (A_max_1) -- (A_max_2) -- (C) -- (B) -- cycle;
                
                \clip (B_min_1) -- (B_min_2) -- (B) -- cycle;
                \clip (B_max_1) -- (B_max_2) -- (C) -- (A) -- cycle;
                
                \clip (C_min_1) -- (C_min_2) -- (C) -- cycle;
                \clip (C_max_1) -- (C_max_2) -- (B) -- (A) -- cycle;
                
                \fill [draw=black, fill=blue!30, fill opacity = .9]   (A)--(B)--(C)--cycle;
                
            \end{scope}  
        
            \draw[-latex, ultra thick] (-0.87,0.7) -- +(0.07,-0.4) node[midway, right] {$x$};
            
        \end{tikzpicture}   
        \fi
        
        \caption{Well-defined optimum.}
        \label{fig:uncertainty_set1}
    \end{subfigure}
    \begin{subfigure}[b]{0.31\textwidth}
        \centering

        \iftikzcompile
        \begin{tikzpicture}[scale=1.5] 

            \pgfmathsetmacro\AngleFuite{150}
            \pgfmathsetmacro\coeffReduc{.8}
            \pgfmathsetmacro\clen{2}
            \pgfmathsinandcos\sint\cost{\AngleFuite}
            
            \begin{scope} [x     = {(\coeffReduc*\cost,-\coeffReduc*\sint)},
                           y     = {(1cm,0cm)}, 
                           z     = {(0cm,1cm)}]
                
                \newcommand\minA{0.15}
                \newcommand\maxA{0.7}
                \newcommand\minB{0.2}
                \newcommand\maxB{0.55}
                \newcommand\minC{0.2}
                \newcommand\maxC{0.75}
        
                \path coordinate (O) at (0,0,0)
                      coordinate (A) at (1,0,0)
                      coordinate (B) at (0,1,0)
                      coordinate (C) at (0,0,1);
                \path coordinate (A_min_1) at (\minA,1-\minA,0)
                      coordinate (A_min_2) at (\minA,0,1-\minA)
                      coordinate (B_min_1) at (1-\minB,\minB,0)
                      coordinate (B_min_2) at (0,\minB,1-\minB)
                      coordinate (C_min_1) at (1-\minC,0,\minC)
                      coordinate (C_min_2) at (0,1-\minC,\minC);
                \path coordinate (A_max_1) at (\maxA,1-\maxA,0)
                      coordinate (A_max_2) at (\maxA,0,1-\maxA)
                      coordinate (B_max_1) at (1-\maxB,\maxB,0)
                      coordinate (B_max_2) at (0,\maxB,1-\maxB)
                      coordinate (C_max_1) at (1-\maxC,0,\maxC)
                      coordinate (C_max_2) at (0,1-\maxC,\maxC);
                
                \draw[gray] (A)--(B)--(C)--cycle;
        
                \draw[fill=gray] (A) circle (1.5pt) node[below, yshift=-0.1cm, xshift=0.2cm] {\small $P_{1}=1$};
                \draw[fill=gray] (B) circle (1.5pt) node[below] {$\small P_{2}=1$};
                \draw[fill=gray] (C) circle (1.5pt) node[right, yshift=0.1cm] {$\small P_{3}=1$};
                
                \draw[thick,-stealth,black] (O) -- (0,0,1);
                \draw[thick,-stealth,black] (O) -- (1,0,0);
                \draw[thick,-stealth,black] (O) -- (0,1,0);
        
                \draw[fill=red] (0.35, 0.2, 0.45) circle (1.5pt) node[red, left] {\small $p^\star$};
                
                \draw[name path=Amin, densely dotted, gray] (A_min_1) -- (A_min_2) node[left,yshift=0.1cm] {\scriptsize $\munderbar{p}_{1}$};
                \draw[name path=Bmin, densely dotted, thick, red] (B_min_2) -- (B_min_1) node[below, xshift=0.2cm] {\scriptsize $\munderbar{p}_{2}$};
                \draw[name path=Cmin, densely dotted, gray] (C_min_1) -- (C_min_2) node[right,yshift=0.1cm] {\scriptsize $\munderbar{p}_{3}$};
                
                \draw[name path=Amax, densely dotted, gray] (A_max_1) -- (A_max_2) node[left,yshift=0.1cm] {\scriptsize $\bar{p}_{1}$};
                \draw[name path=Bmax, densely dotted, gray] (B_max_2) -- (B_max_1) node[below, xshift=0.2cm] {\scriptsize $\bar{p}_{2}$};
                \draw[name path=Cmax, densely dotted, gray] (C_max_1) -- (C_max_2) node[right] {\scriptsize $\bar{p}_{3}$};

                \clip (A_min_1) -- (A_min_2) -- (A) -- cycle;
                \clip (A_max_1) -- (A_max_2) -- (C) -- (B) -- cycle;
                
                \clip (B_min_1) -- (B_min_2) -- (B) -- cycle;
                \clip (B_max_1) -- (B_max_2) -- (C) -- (A) -- cycle;
                
                \clip (C_min_1) -- (C_min_2) -- (C) -- cycle;
                \clip (C_max_1) -- (C_max_2) -- (B) -- (A) -- cycle;
                
                \fill [draw=black, fill=blue!30, fill opacity = .9]   (A)--(B)--(C)--cycle;
                
            \end{scope}  
        
            \draw[-latex, ultra thick] (-0.87,0.7) -- +(0.4,-0.18) node[midway, above] {$x$};
            
        \end{tikzpicture}   
        \fi
        
        \caption{Non-unique optimum.}
        \label{fig:uncertainty_set2}
    \end{subfigure}
    \begin{subfigure}[b]{0.36\textwidth}
        \centering

        \iftikzcompile
        \begin{tikzpicture}[scale=1.5] 

            \pgfmathsetmacro\AngleFuite{150}
            \pgfmathsetmacro\coeffReduc{.8}
            \pgfmathsetmacro\clen{2}
            \pgfmathsinandcos\sint\cost{\AngleFuite}
            
            \begin{scope} [x     = {(\coeffReduc*\cost,-\coeffReduc*\sint)},
                           y     = {(1cm,0cm)}, 
                           z     = {(0cm,1cm)}]
                
                \newcommand\minA{0.15}
                \newcommand\maxA{0.7}
                \newcommand\minB{0.2}
                \newcommand\maxB{0.55}
                \newcommand\minC{0.2}
                \newcommand\maxC{0.65}
        
                \path coordinate (O) at (0,0,0)
                      coordinate (A) at (1,0,0)
                      coordinate (B) at (0,1,0)
                      coordinate (C) at (0,0,1);
                \path coordinate (A_min_1) at (\minA,1-\minA,0)
                      coordinate (A_min_2) at (\minA,0,1-\minA)
                      coordinate (B_min_1) at (1-\minB,\minB,0)
                      coordinate (B_min_2) at (0,\minB,1-\minB)
                      coordinate (C_min_1) at (1-\minC,0,\minC)
                      coordinate (C_min_2) at (0,1-\minC,\minC);
                \path coordinate (A_max_1) at (\maxA,1-\maxA,0)
                      coordinate (A_max_2) at (\maxA,0,1-\maxA)
                      coordinate (B_max_1) at (1-\maxB,\maxB,0)
                      coordinate (B_max_2) at (0,\maxB,1-\maxB)
                      coordinate (C_max_1) at (1-\maxC,0,\maxC)
                      coordinate (C_max_2) at (0,1-\maxC,\maxC);
                
                \draw[gray] (A)--(B)--(C)--cycle;
        
                \draw[fill=gray] (A) circle (1.5pt) node[below, yshift=-0.1cm, xshift=0.2cm] {\small $P_{1}=1$};
                \draw[fill=gray] (B) circle (1.5pt) node[below] {$\small P_{2}=1$};
                \draw[fill=gray] (C) circle (1.5pt) node[right, yshift=0.1cm] {$\small P_{3}=1$};
                
                \draw[thick,-stealth,black] (O) -- (0,0,1);
                \draw[thick,-stealth,black] (O) -- (1,0,0);
                \draw[thick,-stealth,black] (O) -- (0,1,0);
        
                \draw[fill=red] (0.15, 0.2, 0.65) circle (1.5pt) node[red, right] {\small $p^\star$};
                
                \draw[name path=Amin, densely dotted, thick, red] (A_min_1) -- (A_min_2) node[left,yshift=0.1cm] {\scriptsize $\munderbar{p}_{1}$};
                \draw[name path=Bmin, densely dotted, thick, red] (B_min_2) -- (B_min_1) node[below, xshift=0.2cm] {\scriptsize $\munderbar{p}_{2}$};
                \draw[name path=Cmin, densely dotted, gray] (C_min_1) -- (C_min_2) node[right,yshift=0.1cm] {\scriptsize $\munderbar{p}_{3}$};
                
                \draw[name path=Amax, densely dotted, gray] (A_max_1) -- (A_max_2) node[left,yshift=0.1cm] {\scriptsize $\bar{p}_{1}$};
                \draw[name path=Bmax, densely dotted, gray] (B_max_2) -- (B_max_1) node[below, xshift=0.2cm] {\scriptsize $\bar{p}_{2}$};
                \draw[name path=Cmax, densely dotted, thick, red] (C_max_1) -- (C_max_2) node[right] {\scriptsize $\bar{p}_{3}$};

                \clip (A_min_1) -- (A_min_2) -- (A) -- cycle;
                \clip (A_max_1) -- (A_max_2) -- (C) -- (B) -- cycle;
                
                \clip (B_min_1) -- (B_min_2) -- (B) -- cycle;
                \clip (B_max_1) -- (B_max_2) -- (C) -- (A) -- cycle;
                
                \clip (C_min_1) -- (C_min_2) -- (C) -- cycle;
                \clip (C_max_1) -- (C_max_2) -- (B) -- (A) -- cycle;
                
                \fill [draw=black, fill=blue!30, fill opacity = .9]   (A)--(B)--(C)--cycle;
                
            \end{scope}  
        
            \draw[-latex, ultra thick] (-0.87,0.7) -- +(0.07,-0.4) node[midway, right] {$x$};
            
        \end{tikzpicture}   
        \fi
                
        \caption{Too many active constraints.}
        \label{fig:uncertainty_set3}
    \end{subfigure}
    
    \caption{Three polytopic uncertainty sets (blue shade), with the vector $x$, the worst-case points $p^\star$, and the active constraints shown in red.}
    \label{fig:uncertainty_set}
\end{figure*}
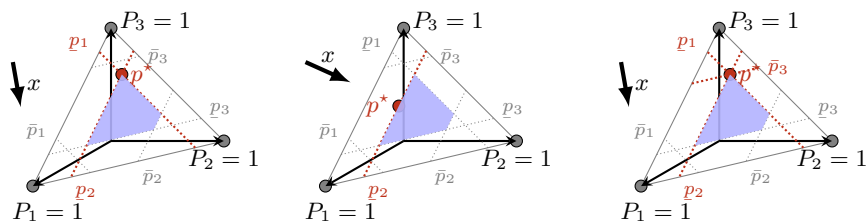

For small changes in the parameters, the point $p^\star$ in \cref{fig:uncertainty_set1} changes smoothly, and its closed-form expression (\ie, the functional form) remains the same.
As such, it feels intuitive that we could apply the methods from \cref{sec:differentiating_pMCs} to compute partial derivatives on the induced \gls{pMC}.
However, this approach does not always work, as illustrated by the following two corner cases.
\begin{enumerate}
    \item Consider \cref{fig:uncertainty_set2}, where the optimization direction defined by $x$ is parallel to one of the facets of the uncertainty set.
    In this case, the worst-case point $p^\star$ is not unique, but an infinitesimal change in the optimization direction $x$ will force the point to one of the vertices again.
    Which point should we choose to obtain the induced \gls{pMC} (and does this choice affect the derivative)?
    \item Consider \cref{fig:uncertainty_set3} with more than $|\States|-1$ active constraints at the point $p^\star$.
    Observe that decreasing $\bar{p}_3$ changes the point $p^\star$ while increasing $\bar{p}_3$ does not.
    In fact, the optimal point $p^\star$ changes \emph{non-smoothly} with the halfspaces of the polytope.
    As a result, also the solution changes non-smoothly, and thus, the derivative is not defined.
    How do we deal with such a situation?
\end{enumerate}
These examples show that computing derivatives via an induced \gls{pMC} by obtaining each point $p^\star_s$ can be tricky or is, in some cases, not possible at all.
In what follows, we present a method that directly derives a set of linear equations to obtain derivatives for \glspl{prMC} (all or only the $k$ highest) based on the solution to the LP in \cref{eq:prmc_LP}, which intrinsically identifies the corner cases above in which the derivative is not defined.

\subsection{Computing derivatives explicitly}

We now develop a dedicated method for identifying if the derivative of the solution function for a \gls{prMC} exists, and if so, to compute this derivative.
Observe from \cref{fig:uncertainty_set} that the point $p^\star$ is uniquely defined and has a smooth derivative only in \cref{fig:uncertainty_set1} with two active constraints.
For only one active constraint (\cref{fig:uncertainty_set2}), the point is \emph{underdetermined}, while for three active constraints (\cref{fig:uncertainty_set3}), the derivative may \emph{not be smooth}.
In the general case, having exactly $n-1$ active constraints (whose facets are nonparallel) is a sufficient condition for obtaining a unique and smoothly changing point $p^\star$ in the $n$-dimensional probability simplex.

\paragraph{Optimal dual variables.}
The optimal dual variables $\alpha_s^\star \geq 0$ for each $s \in \States \setminus \terminalStates$ in \cref{eq:prmc_LP} indicate which constraints of the polytope $A_s[\U] p \leq b_s[\U]$ are active, i.e., for which rows $a_{s,i}[\U]$ of $A_s[\U]$ it holds that $a_{s,i}[\U] p^\star = b_s[\U]$.
Specifically, a value of $\alpha_{s,i} > 0$ implies that the $i^\text{th}$ constraint is active, and $\alpha_{s,i} = 0$ indicates a nonactive constraint~\cite{DBLP:books/cu/BV2014}.
We define $E_s = [e_1, \ldots, e_{m_s}] \in \{0, 1\}^{m_s}$ as a vector whose binary values $e_i \,\forall i \in \{1,\ldots,m_s\}$ are given as $e_i = \llbracket\alpha^\star_{s,i} > 0\rrbracket$.\footnote{We use Iverson-brackets: $\llbracket x\rrbracket = 1$ if $x$ is true and $\llbracket x \rrbracket=0$ otherwise.} 
Moreover, denote by $\diag{E_s}$ the matrix with $E_s$ on the diagonal and zeros elsewhere.
We reduce the LP in \cref{eq:prmc_LP} to a system of linear equations that encodes only the constraints that are active under the worst-case point $p^\star_s$ for each $s \in \States \setminus \terminalStates$:
\begin{subequations}
\label{eq:prmc_LP_reduced}
\begin{alignat}{2}
    & x_s = 0, &&\forall s \in \terminalStates
    \label{eq:prmc_LP_red_constr1}
    \\
    & 
    x_s = r_s - \left( b_{s}[\U]^\top \diag{E_s} \alpha_{s} + \beta_{s} \right),
    \enskip &&\forall s \in \States \setminus \terminalStates
    \label{eq:prmc_LP_red_constr2}
    \\
    & A_{s}[\U]^\top \diag{E_s} \alpha_{s} + x_{\post{s}} + \beta_{s}\mathbbm{1} = 0, \quad \alpha_{s} \geq 0, \quad && \forall s \in \States \setminus \terminalStates.
    \label{eq:prmc_LP_red_constr3}
\end{alignat}%
\end{subequations}%
\paragraph{Differentiation.}
However, when does \cref{eq:prmc_LP_reduced} have a (unique) optimal solution?
To provide some intuition, let us write the equation system in matrix form, i.e., $C \mleft[ \begin{array}{ccc} x & \alpha & \beta \end{array} \mright]^\top = d$, where we omit an explicit definition of matrix $C$ and vector $d$ for brevity.
It is apparent that if matrix $C$ is nonsingular, then \cref{eq:prmc_LP_reduced} has a unique solution.
This requires matrix $C$ to be square, which is achieved if, for each $s \in \States \setminus \terminalStates$, we have $|\post{s}| = \sum{E}_s + 1$.
In other words, the number of successor states of $s$ is equal to the number of active constraints of the polytope plus one.
This confirms our previous intuition from \cref{sec:differentiating_prMCs:intuition} on a polytope for $|\post{s}| = 3$ successor states, which required $\sum_{i = 1}^{m_s} E_i = 2$ active constraints. 

Let us formalize this intuition about computing derivatives for \glspl{prMC}. We can compute the derivative of the solution $x^\star$ by differentiating the equation system in \cref{eq:prmc_LP_reduced} through the product rule, in a very similar manner to the approach in \cref{sec:differentiating_pMCs}.
We state this key result in the following theorem.
\begin{theorem}
    \label{thm:derivative_prmc}
    Given a \gls{prMC} $\prmc$ and an instantiation $\U$, compute $x^\star, \alpha^\star, \beta^\star$ for \cref{eq:prmc_LP} and choose a parameter $\param_i \in \Param$.
    The partial derivatives $\parder{x}{u(\param_i)}$, $\parder{\alpha}{u(\param_i)}$, and $\parder{\beta}{u(\param_i)}$ are obtained as the solution to the linear equation system
    \begin{subequations}
        \label{eq:prmc_LP_deriv}
        \begin{alignat}{2}
        & \parder{x_s}{u(\param_i)} = 0, && \qquad\forall s \in \terminalStates
        \label{eq:prmc_LP_deriv_constr1}
        \\
        & 
        \parder{x_s}{u(\param_i)} + b_s[\U]^\top \diag{E_s} \parder{\alpha_s}{u(\param_i)} + \parder{\beta_s}{u(\param_i)} &&= -(\alpha_s^\star)^\top \diag{E_s} \parder{b_s[\U]}{u(\param_i)},
        \label{eq:prmc_LP_deriv_constr2}
        \\
        & 
        && 
        \qquad\forall s \in \States \setminus \terminalStates
        \nonumber
        \\    
        & A_s[\U]^\top \diag{E_s} \parder{\alpha_s}{u(\param_i)} + \parder{x_{\post{s}}}{u(\param_i)} + \parder{\beta_s}{u(\param_i)} \mathbbm{1} &&= -(\alpha^\star_s)^\top \diag{E_s} \parder{A_s[\U]}{u(\param_i)}, 
        \label{eq:prmc_LP_deriv_constr3}
        \\
        & 
        && 
        \qquad\forall s \in \States \setminus \terminalStates.
        \nonumber
        \end{alignat}%
    \end{subequations}%
\end{theorem}%
The proof follows from applying the product rule to \cref{eq:prmc_LP_reduced} and is provided in \cref{app:proof:thm2}.
To compute the derivative for a parameter $\param_i \in \Param$, we thus solve a system of linear equations of size $|\States| + \sum_{s \in \States \setminus \terminalStates}{|\post{s}|}$.
Using \cref{thm:derivative_prmc}, we obtain sufficient conditions for the solution function to be differentiable.
\begin{lemma}
    \label{lemma:prmc:differentiability}
    Write the linear equation system in \cref{eq:prmc_LP_deriv} in matrix form, i.e.,
    \begin{equation}
        C \mleft[ \begin{array}{ccc} \parder{x}{u(\param_i)}, \parder{\alpha}{u(\param_i)}, \parder{\beta}{u(\param_i)} \end{array} \mright]^\top = d,
    \end{equation}
    for $C \in \Real^{q \times q}$ and $d \in \Real^q$, $q = |\States| + \sum_{s \in \States \setminus \terminalStates}{|\post{s}|}$, which are implicitly given by \cref{eq:prmc_LP_deriv}.
    The solution function $\ExpRR[\U]$ is differentiable at instantiation $\U$ if matrix $C$ is nonsingular, in which case we obtain $(\parder{\ExpRR}{u(\param_i)})[\U] = \sI^\top \parder{x}{u(\param_i)}$.
\end{lemma}

\begin{proof}
The partial derivative of the solution function is $\parder{\ExpRR}{u(\param_i)}[\U] = \sI^\top \parder{x^\star}{u(\param_i}$, where $\parder{x^\star}{u(\param_i}$ is (a part of) the solution to \cref{eq:prmc_LP_reduced}.
Thus, the solution function is differentiable if there is a (unique) solution to \cref{eq:prmc_LP_reduced}, which is guaranteed if matrix $C$ is nonsingular.
Thus, the claim in \cref{lemma:prmc:differentiability} follows. \qed
\end{proof}

\paragraph{Algorithm for \cref{prob:formal}.}
We use \cref{thm:derivative_prmc} to solve \cref{prob:formal} for \glspl{prMC}, similarly as for \glspl{pMC}.
Given a \gls{prMC} $\prmc$ and an instantiation $\U$, we first solve \cref{eq:prmc_LP} to obtain $x^\star, \alpha^\star, \beta^\star$.
Second, we use $\alpha^\star_s$ to compute the vector $E_s$ of active constraints for each $s \in \States \setminus \terminalStates$.
Third, for every parameter $\param \in \Param$, we solve the equation system in \cref{eq:prmc_LP_deriv}.
Thus, to compute the gradient of the solution function, we solve one LP and $|\Param|$ linear equation systems.

\subsection{Computing $k$-highest derivatives}
We directly apply the same procedure from \cref{subsec:pmc_faster} to compute the parameters with the $k \leq \ell$ highest derivatives.
As for \glspl{pMC}, we can compute the $k$ highest derivatives by solving a MILP encoding the equation system in \cref{eq:prmc_LP_deriv} for every parameter $\param \in \Param$, which we present in \cref{app:prmc_derivatives} for brevity.
This MILP has the same structure as \cref{eq:pMC_importance_MIP}, and thus we may apply the same linear relaxation to obtain an LP with the guarantees as stated in \cref{thm:LP_relaxation}.
In other words, solving the LP relaxation yields the set $\Param^\star$ of parameters with maximal derivatives as in \cref{prop:pMC}.
This set $\Param^\star$ is a solution to \cref{prob:formal2} for \glspl{prMC}.
\section{Numerical Experiments}
\label{sec:experiments}

We perform experiments to answer the following questions about our approach:
\begin{enumerate}
    \item Is it feasible (in terms of computational complexity and runtimes) to compute all derivatives, in particular compared to computing (robust) solutions?
    \item How does computing only the $k$ highest derivatives compare to computing  all derivatives?
    \item Can we apply our approach to effectively determine for which parameters to sample in a learning framework?
\end{enumerate}
Let us briefly summarize the computations involved in answering these questions.
First of all, computing the solution $\ExpR (\U)$ for a \gls{pMC}, which is defined in \cref{eq:solution_pmc}, means solving the linear equation system in \cref{eq:pMC_eqsys}.
Similarly, computing the robust solution $\ExpRR (\U)$ for a \gls{prMC} means solving the LP in \cref{eq:prmc_LP}.
Then, solving \cref{prob:formal}, i.e., computing all $|\Param|$ partial derivatives, amounts to solving a linear equation system for each parameter $\param \in \Param$ (namely, \cref{eq:pMC_eqsys} for a \gls{prMC} and \cref{eq:prmc_LP_deriv} for a \gls{prMC}).
In contrast, solving \cref{prob:formal2}, i.e., computing a subset $V^\star$ of parameters with maximal (or minimal) derivative, means for a \gls{pMC} that we solve the LP in \cref{eq:pMC_importance_LP} (or the equivalent LP for a \gls{prMC}) and thereafter extract the subset of $V^\star$ parameters using \cref{prop:pMC}.

\paragraph{Problem~3: Computing the $k$-highest derivatives.}
A solution to \cref{prob:formal2} is a set $\Param^\star$ of $k$ parameters but does not include the computation of the derivatives. 
However, it is straightforward to also obtain the actual derivatives $\left( \parder{\ExpR}{u(\param)} \right)[\U]$ for each parameter $\param \in \Param^\star$.
Specifically, we solve \cref{prob:formal} for the $k$ parameters in $\Param^\star$, such that we obtain the partial derivatives for all $\param \in \Param^\star$.
We remark that, for $k=1$, the derivative follows directly from the optimal value $\sI^\top y^+$ of the LP in \cref{eq:pMC_importance_LP}, so this additional step is not necessary.
We will refer to computing the actual values of the $k$ highest derivatives as \emph{Problem~3}.

\paragraph{Setup.}
We implement our approach in Python 3.10, using \storm~\cite{StormSTTT} to parse \glspl{pMC}, Gurobi~\cite{gurobi} to solve LPs, and the SciPy sparse solver to solve equation systems.
All experiments run on a computer with a 4GHz Intel Core i9 CPU and 64 GB RAM, with a timeout of one hour.
Our implementation is available at \url{https://doi.org/10.5281/zenodo.7864260}.

\paragraph{Grid world benchmarks.}
We use scaled versions of the grid world from the example in \cref{sec:motivation} with over a million states and up to $10\,000$ terrain types.
The vehicle only moves right or down, both with 50\% probability (wrapping around when leaving the grid).
Slipping only occurs when moving down and (slightly different from the example in \cref{sec:motivation}) means that the vehicle moves \emph{two cells instead of one}.
We obtain between $N=500$ and $1\,000$ samples of each slipping probability.
For the \glspl{pMC}, we use maximum likelihood estimation ($\frac{\bar{p}}{N}$, with $\bar{p}$ the sample mean) obtained from these samples as probabilities, whereas, for the \glspl{prMC}, we infer probability intervals using Hoeffding's inequality (see Q3 for details).

\paragraph{Benchmarks from literature.}
We also use several instances of parametric extensions of \glspl{MC} and \glspl{MDP} from standard benchmark suits~\cite{DBLP:conf/tacas/HartmannsKPQR19,DBLP:conf/cav/KwiatkowskaNP11}.
We also use \gls{pMC} benchmarks from~\cite{DBLP:journals/tac/CubuktepeJJKT22,DBLP:journals/sttt/BadingsCJJKT22} as these models have more parameters than the traditional benchmarks. 
We extend these benchmarks to \glspl{prMC} by constructing probability intervals around the \gls{pMC}'s probabilities.

\paragraph{Results.}
The results for all benchmarks are shown in \cref{tab:grid_world_results_full,tab:benchmarks} in \cref{app:experiments}.

\subsection*{Q1. Computing solutions vs. derivatives}
We investigate whether computing derivatives is feasible on p(r)MCs. 
In particular, we compare the computation times for computing derivatives on p(r)MCs (Problems 1 and 3) with the times for computing the solution for these models.

\newcommand{\engine}[1]{\textsf{#1}}
\newlength{\scatterplotsize}
\setlength{\scatterplotsize}{.48\textwidth}

\begin{figure}[t!]
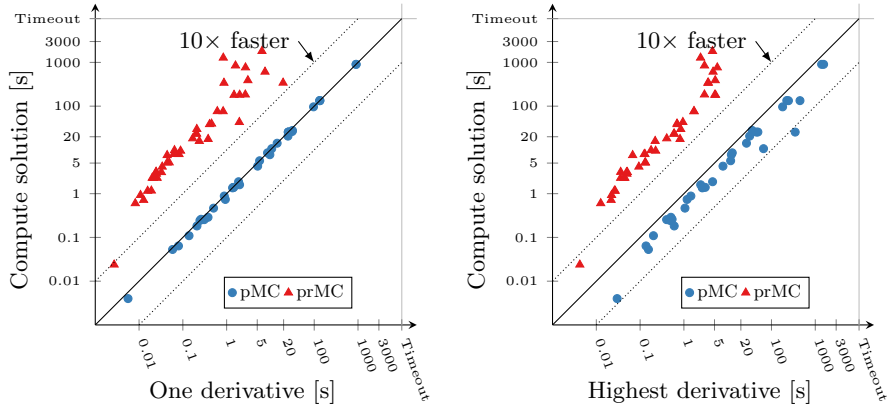

    \centering
    \iftikzcompile
        \scatterplotstormB{tikz/scatter_cameraready/sol_vs_deriv.csv}{OneDeriv}{One derivative [s]}{Verify}{Compute solution [s]}{Type}
        \scatterplotstormB{tikz/scatter_cameraready/sol_vs_deriv.csv}{Highest}{Highest derivative [s]}{Verify}{Compute solution [s]}{Type}
    \fi
    \caption{Runtimes (log-scale) for computing a single derivative (left, Problem~1) or the highest derivative (right, Problem~3), vs. computing the solution $\ExpR[\U]$/$\ExpRR[\U]$.
    }
    \label{fig:scatter_solutions}
\end{figure}

In \cref{fig:scatter_solutions}, we show for all benchmarks the times for computing the solution (defined in \cref{eq:solution_pmc,eq:solution_prmc}), versus computing either a single derivative for \cref{prob:formal} (left) or the highest derivative of all parameters resulting from Problem~3 (right).
A point $(x,y)$ in the left plot means that computing a single derivative took $x$ seconds while computing the solution took $y$ seconds. 
A line above the (center) diagonal means we obtained a speed-up over the time for computing the solution; a point over the upper diagonal indicates a $10\times$ speed-up or larger.

\paragraph{One derivative.}
The left plot in \cref{fig:scatter_solutions} shows that, for \glspl{pMC}, the times for computing the solution and a single derivative are approximately the same.
This is expected since both problems amount to solving a single equation system with $|\States|$ unknowns.
Recall that, for \glspl{prMC}, computing the solution means solving the LP in \cref{eq:prmc_LP}, while for derivatives we solve an equation system.
Thus, computing a derivative for a \gls{prMC} is relatively cheap compared to computing the solution, which is confirmed by the results in \cref{fig:scatter_solutions}.

\paragraph{Highest derivative.}
The right plot in \cref{fig:scatter_solutions} shows that, for \glspl{pMC}, computing the highest derivative is slightly slower than computing the solution (the LP to compute the highest derivative takes longer than the equation system to compute the solution).
On the other hand, computing the highest derivative for a \gls{prMC} is still cheap compared to computing the solution.
Thus, if we are using a \gls{prMC} anyways, computing the derivatives is relatively cheap.

\subsection*{Q2. Runtime improvement of computing only $k$ derivatives}
We want to understand the computational benefits of solving Problem~3 over solving \cref{prob:formal}. 
For Q2, we consider all models with $|\Param| \geq 10$ parameters.

{
\setlength{\tabcolsep}{3pt}
\begin{table*}[t!]

\centering
\caption{Model sizes, runtimes, and derivatives for selection of grid world models.}

\begin{threeparttable}

\scalebox{0.75}{
\setlength{\tabcolsep}{3pt}
\begin{tabular}{lrrrrrrrrrr}
\toprule
\multicolumn{4}{c}{{Model statistics}} 
& \multicolumn{2}{c}{{Verifying}} 
& \multicolumn{1}{c}{{Problem 1}} 
& \multicolumn{2}{c}{{Problem 3}} 
& \multicolumn{2}{c}{{Derivatives}} \\
  \cmidrule(lr){1-4} \cmidrule(lr){5-6} \cmidrule(lr){7-7} \cmidrule(lr){8-9} \cmidrule(lr){10-11}
Type & $|\States|$ & $|\Param|$ & \#trans & $\ensuremath{\mathsf{sol}_{(R)}}[\U]$ & Time [s] & All derivs. [s] & $k=1$ [s] & $k=10 [s]$ & Highest & Error \% \\
\midrule
pMC &    5000 &         50 &       14995 &     5.07 &             1.39 &                        3.32 &               2.64 &                2.69 &         1.54e+00 &          0.0 \\
pMC &    5000 &        100 &       14995 &     5.05 &             1.36 &                        4.17 &               2.63 &                2.66 &         1.28e+00 &          0.0 \\
pMC &    5000 &        921 &       14995 &     4.93 &             1.87 &                       19.92 &               4.52 &                2.87 &         1.20e+00 &          0.0 \\
pMC &   80000 &        100 &      239995 &     8.01 &            25.54 &                       98.47 &              45.18 &               46.87 &         1.95e+00 &          0.0 \\
pMC &   80000 &       1000 &      239995 &     8.01 &            25.64 &                      612.97 &              48.92 &               58.20 &         2.08e+00 &          0.0 \\
pMC &   80000 &       9831 &      239995 &     7.93 &            25.52 &                    5,650.25 &             347.76 &            1,343.59 &         2.10e+00 &          0.0 \\
pMC & 1280000 &        100 &     3839995 &    12.90 &           902.52 &                    4,747.43 &           1,396.51 &            1,507.77 &         3.32e+00 &          0.0 \\
pMC & 1280000 &       1000 &     3839995 &    12.79 &           902.67 &                   37,078.12 &           1,550.45 &            1,617.27 &         3.18e+00 &          0.0 \\
pMC & 1280000 &      10000 &     3839995 &   Timeout\tnote{b} &           --- &                   --- &           --- &            --- &         --- &          --- \\
\midrule
prMC &    5000 &        100 &       14995 &   136.07 &            23.46 &                        3.55 &               0.60 &                1.58 &        -1.26e-02 &         -0.0 \\
prMC &    5000 &        921 &       14995 &   138.74 &            29.82 &                       25.23 &               0.85 &                1.09 &        -4.44e-03 &         -0.0 \\
prMC &   20000 &        100 &       59995 & 2,789.77 &         1,276.43 &                       15.68 &               2.40 &                2.70 &        -4.96e-01 &         -0.1 \\
prMC &   20000 &       1000 &       59995 & 2,258.41 &           339.96 &                      159.70 &               3.53 &                4.09 &        -9.51e-02 &         -0.0 \\
prMC &   80000 &       100 &       239995 & Timeout\tnote{b} &           --- &                      --- &               --- &                --- &        --- &         --- \\
\bottomrule
\end{tabular}

}
\begin{tablenotes}
        \raggedright
        \item[a] \tableextrapolationtext
        \item[b] \tabletimeouttext
\end{tablenotes}
\end{threeparttable}

\label{tab:grid_world_results}
\end{table*}
}

An excerpt of results for the grid world benchmarks is presented in \cref{tab:grid_world_results}.
Recall that, after obtaining the (robust) solution, solving \cref{prob:formal} amounts to solving $|\Param|$ linear equation systems, whereas Problem 3 involves solving a single LP and $k$ equations systems.
From \cref{tab:grid_world_results}, it is clear that computing $k$ derivatives is orders of magnitudes faster than computing all $|\Param|$ derivatives, especially if the total number of parameters is high.

\begin{figure}[t!]
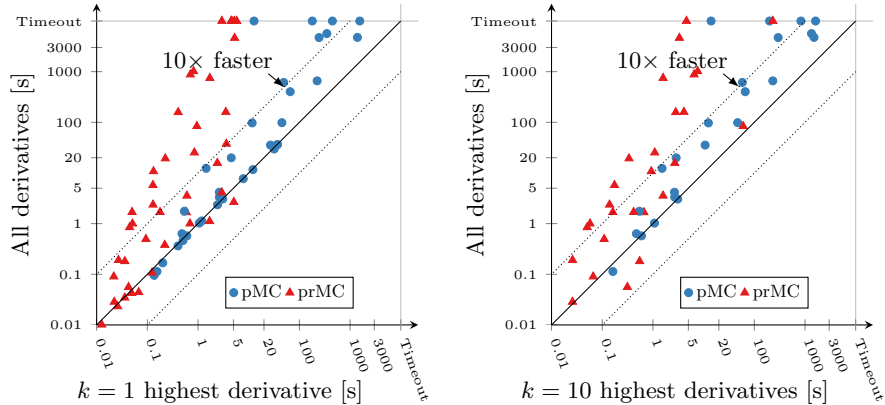

    \centering
    \iftikzcompile
        \scatterplotstorm{tikz/scatter_cameraready/allClean_k=1.csv}{Prob1}{$k=1$ highest derivative [s]}{Prob2}{All derivatives [s]}{Type}
        \scatterplotstorm{tikz/scatter_cameraready/allClean_k=10.csv}{Prob1}{$k=10$ highest derivatives [s]}{Prob2}{All derivatives [s]}{Type}
    \fi
    \caption{Runtimes (log-scale) for computing the highest (left) or $10$ highest (right) derivatives (Problem~3), versus computing all derivatives (\cref{prob:formal}).
    }
    \label{fig:run_time}
\end{figure}

We compare the runtimes for computing all derivatives (\cref{prob:formal}) with computing only the $k=1$ or $10$ highest derivatives (Problem~3).
The left plot of \cref{fig:run_time} shows the runtimes for $k=1$, and the right plot for the $k=10$ highest derivatives.
The interpretation for \cref{fig:run_time} is the same as for \cref{fig:scatter_solutions}.
From \cref{fig:run_time}, we observe that computing only the $k$ highest derivatives generally leads to significant speed-ups, often of more than $10$ times (except for very small models).
Moreover, the difference between $k=1$ and $k=10$ is minor, showing that retrieving the actual derivatives after solving \cref{prob:formal2} is relatively cheap.

\paragraph{Numerical stability.}
While our algorithm is exact, our implementation uses floating-point arithmetic for efficiency.
To evaluate the numerical stability, we compare the highest derivatives (solving Problem~3 for $k=1$) with an empirical approximation of the derivative obtained by perturbing the parameter by $\num{1e-3}$.
The difference (column \emph{`Error. \%'} in \cref{tab:grid_world_results,tab:grid_world_results_full}) between both is marginal, indicating that our implementation is sufficiently numerically stable to return accurate derivatives.

\subsection*{Q3. Application in a learning framework}
Reducing the sample complexity is a key challenge in learning under uncertainty~\cite{kakade2003sample,Moerland2020modelbasedRL}.
In particular, learning in stochastic environments is very data-intensive, and realistic applications tend to require millions of samples to provide tight bounds on measures of interest~\cite{DBLP:conf/nips/BuckmanHTBL18}.
Motivated by this challenge, we apply our approach in a learning framework to investigate if derivatives can be used to effectively guide exploration, compared to alternative exploration strategies.

\paragraph{Models.}
We consider the problem of where to sample in 1) a slippery grid world with $|\States| = 800$ and $|\Param| = 100$ terrain types, and 2) the drone benchmark from \cite{DBLP:journals/tac/CubuktepeJJKT22} with $|\States| = 4\,179$ and $|\Param| = 1\,053$ parameters.
As in the motivating example in \cref{sec:motivation}, we learn a model of the unknown \gls{MC} in the form of a \gls{prMC}, where the parameters are the sample sizes for each parameter.
We assume access to a model that can arbitrarily sample each parameter (i.e., the slipping probability in the case of the grid world).
We use an initial sample size of $N_i=100$ for each parameter $i \in \{1,\ldots,|\Param|\}$, from which we infer a $\beta = 0.9$ (90\%) confidence interval using Hoeffding's inequality.
The interval for parameter $i$ is $[\hat{p}_i - \epsilon_i, \hat{p}_i + \epsilon_i]$, with $\hat{p}_i$ the sample mean and $\epsilon_i = \sqrt{\frac{\log{2} - \log{(1-\beta)}}{2N}}$ (see, e.g.,~\cite{DBLP:books/daglib/Boucheron2013} for details).

\paragraph{Learning scheme.}
We iteratively choose for which parameter $\param_i \in \Param$ to obtain $25$ (for the grid world) or $250$ (for the drone) additional samples.
We compare four strategies for choosing the parameter $\param_i$ to sample for: 1) with highest derivative, i.e., solving Problem~3 for $k=1$; 2) with biggest interval width $\epsilon_i$; 3) uniformly; and 4) sampling according to the expected number of visits times the interval width (see \cref{app:experiments:learning} for details).
After each step, we update the robust upper bound on the solution for the \gls{prMC} with the additional samples.

\paragraph{Results.}
The upper bounds on the solution for each sampling strategy, as well as the solution for the \gls{MC} with the true parameter values, are shown in \cref{fig:learning}.
For both benchmarks, our derivative-guided sampling strategy converges to the true solution faster than the other strategies.
Notably, our derivative-guided strategy accounts for both the uncertainty and importance of each parameter, which leads to a lower sample complexity required to approach the true solution.

\begin{figure}[t!]
    \subfloat[Slippery grid world.]{%
    \footnotesize\begin{tikzpicture}
  \begin{axis}[
      width=.48\linewidth,
      height=5cm,
      ymajorgrids,
      grid style={dashed,gray!40},
      xlabel=Steps (of $25$ samples each),
      ylabel=Robust solution,
      xmin=0,
      xmax=1000,
      ymin=0,
      ymax=200,
      xtick={0,250,...,1000},
      xtick pos=left,
      ytick={0,50,100,...,250},
      every axis plot/.append style={line width=1pt},
      legend cell align={left},
      legend columns=1,
      legend image post style={scale=0.6},
      legend style={nodes={scale=0.9, transform shape},
                    anchor=north east, 
                    column sep=0ex,}
    ]

    \addplot[color=color1] table[x=x, y=derivative_mean, col sep=semicolon] {tikz/learning/cameraready/learning_gridworld.csv};
    
    \addplot [name path=upper_derivative, draw=none] table[x=x, y=derivative_max, col sep=semicolon, forget plot] {tikz/learning/cameraready/learning_gridworld.csv};
    \addplot [name path=lower_derivative, draw=none] table[x=x, y=derivative_min, col sep=semicolon, forget plot] {tikz/learning/cameraready/learning_gridworld.csv};
    \addplot [fill=color1!40, forget plot] fill between[of=upper_derivative and lower_derivative];

    \addplot[color=color2, dashed] table[x=x, y=samples_mean, col sep=semicolon] {tikz/learning/cameraready/learning_gridworld.csv};
    
    \addplot [name path=lower_samples, draw=none] table[x=x, y=samples_min, col sep=semicolon, forget plot] {tikz/learning/cameraready/learning_gridworld.csv};
    \addplot [name path=upper_samples, draw=none] table[x=x, y=samples_max, col sep=semicolon, forget plot] {tikz/learning/cameraready/learning_gridworld.csv};
    \addplot [fill=color2!20, forget plot, fill opacity=1.0] fill between[of=lower_samples and upper_samples];
    
    \addplot[color=color3, dotted] table[x=x, y=random_mean, col sep=semicolon] {tikz/learning/cameraready/learning_gridworld.csv};
    
    \addplot [name path=lower_random, draw=none] table[x=x, y=random_min, col sep=semicolon, forget plot] {tikz/learning/cameraready/learning_gridworld.csv};
    \addplot [name path=upper_random, draw=none] table[x=x, y=random_max, col sep=semicolon, forget plot] {tikz/learning/cameraready/learning_gridworld.csv};
    \addplot [fill=color3!40, forget plot] fill between[of=lower_random and upper_random];

    \addplot[color=color4, dashdotted] table[x=x, y=expVisits_sampling_mean, col sep=semicolon] {tikz/learning/cameraready/learning_gridworld.csv};
    
    \addplot [name path=lower_expVisits, draw=none] table[x=x, y=expVisits_sampling_min, col sep=semicolon, forget plot] {tikz/learning/cameraready/learning_gridworld.csv};
    \addplot [name path=upper_expVisits, draw=none] table[x=x, y=expVisits_sampling_max, col sep=semicolon, forget plot] {tikz/learning/cameraready/learning_gridworld.csv};
    \addplot [fill=color4!40, forget plot] fill between[of=lower_expVisits and upper_expVisits];

    \draw [densely dashed, very thick, gray] (axis cs:0,10) -- (axis cs:1000,10) node[font={\footnotesize}, pos=0.23, above, black] {True solution};
    
    \legend{{Derivative},{Interval}, {Uniform}, {ExpVisits*Width}}
  \end{axis}
\end{tikzpicture}
    \label{fig:learning_a}}
    \subfloat[Drone motion planning.]{%
    \footnotesize\begin{tikzpicture}
  \begin{axis}[
      width=.48\linewidth,
      height=5cm,
      ymajorgrids,
      grid style={dashed,gray!40},
      xlabel=Steps (of $250$ samples each),
      ylabel=Robust solution,
      xmin=0,
      xmax=10000,
      ymin=0.08,
      ymax=0.45,
      xtick={0,2500,...,10000},
      xtick pos=left,
      xticklabel style={
        /pgf/number format/fixed,
        /pgf/number format/precision=1
      },
      scaled x ticks=false,
      ytick={0,0.1,...,0.5},
      every axis plot/.append style={line width=1pt},
      legend cell align={left},
      legend columns=1,
      legend image post style={scale=0.6},
      legend style={nodes={scale=0.9, transform shape},
                    anchor=north east, 
                    column sep=0ex,}
    ]

    \addplot[color=color1] table[x=x, y=derivative_mean, col sep=semicolon] {tikz/learning/cameraready/learning_drone.csv};
    
    \addplot [name path=upper_derivative, draw=none] table[x=x, y=derivative_max, col sep=semicolon, forget plot] {tikz/learning/cameraready/learning_drone.csv};
    \addplot [name path=lower_derivative, draw=none] table[x=x, y=derivative_min, col sep=semicolon, forget plot] {tikz/learning/cameraready/learning_drone.csv};
    \addplot [fill=color1!40, forget plot] fill between[of=upper_derivative and lower_derivative];

    \addplot[color=color2, dashed] table[x=x, y=samples_mean, col sep=semicolon] {tikz/learning/cameraready/learning_drone.csv};
    
    \addplot [name path=lower_samples, draw=none] table[x=x, y=samples_min, col sep=semicolon, forget plot] {tikz/learning/cameraready/learning_drone.csv};
    \addplot [name path=upper_samples, draw=none] table[x=x, y=samples_max, col sep=semicolon, forget plot] {tikz/learning/cameraready/learning_drone.csv};
    \addplot [fill=color2!20, forget plot, fill opacity=1.0] fill between[of=lower_samples and upper_samples];
    
    \addplot[color=color3, dotted] table[x=x, y=random_mean, col sep=semicolon] {tikz/learning/cameraready/learning_drone.csv};
    
    \addplot [name path=lower_random, draw=none] table[x=x, y=random_min, col sep=semicolon, forget plot] {tikz/learning/cameraready/learning_drone.csv};
    \addplot [name path=upper_random, draw=none] table[x=x, y=random_max, col sep=semicolon, forget plot] {tikz/learning/cameraready/learning_drone.csv};
    \addplot [fill=color3!40, forget plot] fill between[of=lower_random and upper_random];

    \addplot[color=color4, dashdotted] table[x=x, y=expVisits_sampling_mean, col sep=semicolon] {tikz/learning/cameraready/learning_drone.csv};
    
    \addplot [name path=lower_expVisits, draw=none] table[x=x, y=expVisits_sampling_min, col sep=semicolon, forget plot] {tikz/learning/cameraready/learning_drone.csv};
    \addplot [name path=upper_expVisits, draw=none] table[x=x, y=expVisits_sampling_max, col sep=semicolon, forget plot] {tikz/learning/cameraready/learning_drone.csv};
    \addplot [fill=color4!40, forget plot] fill between[of=lower_expVisits and upper_expVisits];

    \draw [densely dashed, very thick, gray] (axis cs:0,0.113) -- (axis cs:10000,0.113) node[font={\footnotesize}, pos=0.23, above, black] {True solution};
    
    \legend{{Derivative},{Interval}, {Uniform}, {ExpVisits*Width}}
  \end{axis}
\end{tikzpicture}
    \label{fig:learning_b}}
    \caption{Robust solutions for each sampling strategy in the learning framework for the grid world (a) and drone (b) benchmarks.
    Averages values of 10 (grid world) or 5 (drone) repetitions are shown, with shaded areas the min/max.}   
    \label{fig:learning}
\end{figure}
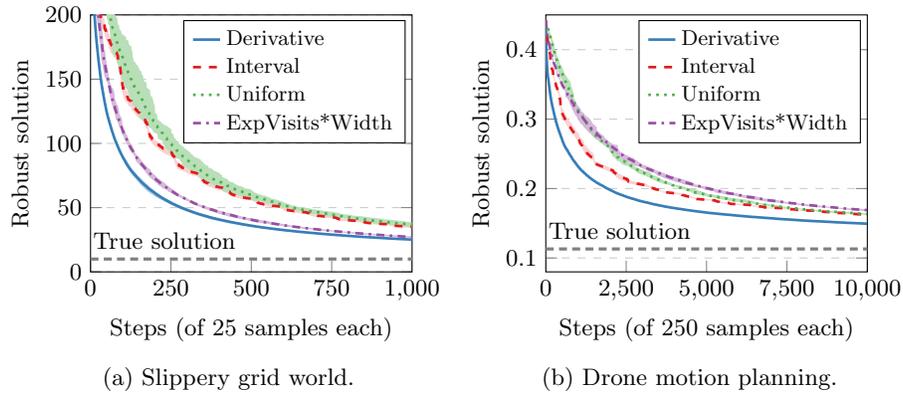
\section{Related Work}
\label{sec:related_work}
We discuss related work in three areas: pMCs, their extension to parametric interval Markov chains (piMCs), and general sensitivity analysis methods. 

\paragraph{Parametric Markov chains.}
\glspl{pMC}~\cite{DBLP:conf/ictac/Daws04,DBLP:journals/fac/LanotteMT07}
have traditionally been studied in terms of computing the solution function~\cite{DBLP:journals/sttt/HahnHZ11,DBLP:conf/cav/DehnertJJCVBKA15,DBLP:journals/iandc/BortolussiMS16,DBLP:journals/tse/FilieriTG16,DBLP:conf/icse/FangCGA21}.
Much recent literature considers synthesis (find a parameter valuation such that a specification is satisfied) or verification (prove that all valuations satisfy a specification).
We refer to~\cite{DBLP:conf/birthday/0001JK22} for a recent overview. 
For our paper, particularly relevant are~\cite{DBLP:conf/tacas/SpelJK21}, which checks whether a derivative is positive (for all parameter valuations), and~\cite{DBLP:conf/vmcai/HeckSJMK22}, which solves parameter synthesis via gradient descent. 
We note that all these problems are (co-)ETR complete~\cite{DBLP:journals/jcss/JungesK0W21} and that the solution function is exponentially large in the number of parameters~\cite{DBLP:journals/iandc/BaierHHJKK20}, whereas we consider a polynomial-time algorithm. 
Furthermore, practical \emph{verification} procedures for uncontrollable parameters (as we do) are limited to less than 10 parameters.
Parametric verification is used in~ \cite{DBLP:conf/qest/PolgreenWHA17} to guide model refinement by detecting for which parameter values a specification is satisfied.
In contrast, we consider slightly more conservative \glspl{rMC} and aim to stepwise optimize an objective. 
Solution functions also provide an approach to compute and refine confidence intervals~\cite{DBLP:journals/tr/CalinescuGJPRT16}; however, the size of the solution function hampers scalability.

\paragraph{Parametric interval Markov chains (piMCs).}
While \glspl{prMC} have, to the best of our knowledge, not been studied, their slightly more restricted version are piMCs.
In particular, piMCs have interval-valued transitions with parametric bounds.
Work on piMCs falls into two categories. First, \emph{consistency}~\cite{DBLP:conf/vmcai/DelahayeLP16,DBLP:conf/forte/PetrucciP18}: is there a parameter instantiation such that the (reachable fragment of the) induced interval MC contains valid probability distributions?
Second, parameter synthesis for quantitative and qualitative reachability in piMCs with up to 12 parameters~\cite{DBLP:journals/tcs/BartDFLMT18}.

\paragraph{Perturbation analysis.}
Perturbation analysis considers the change in solution by any perturbation vector $X$ for the parameter instantiation, whose norm is upper bounded by $\delta$, i.e., $||X|| \leq \delta$ (or conversely, which $\delta$ ensures the solution perturbation is below a given maximum). 
Likewise,~\cite{DBLP:conf/rp/Chonev19} uses the distance between two instantiations of a \gls{pMC} (called augmented interval MC) to bound the change in reachability probability.
Similar analyses exist for stationary distributions~\cite{abbas2016critical}. 
These problems are closely related to the verification problem in \glspl{pMC} and are equally (in)tractable if there are dependencies over multiple parameters. 
To improve tractability, a follow-up~\cite{DBLP:journals/tse/SuFCR16} derives asymptotic bounds based on first or second-order Taylor expansions. 
Other approaches to perturbation analysis analyze individual paths of a system~
\cite{DBLP:journals/orl/FuH94,DBLP:journals/tac/CaoC97,DBLP:journals/tcst/CaoW98}.
Sensitivity analysis in (parameter-free) imprecise \glspl{MC}, a variation to \glspl{rMC}, is thoroughly studied in~\cite{DBLP:conf/uai/CoomanHQ08}.

\paragraph{Exploration in learning.}
Similar to Q3 in \cref{sec:experiments}, determining where to sample is relevant in many learning settings.
Approaches such as probably approximately correct (PAC) statistical model checking~\cite{DBLP:conf/cav/AshokKW19,DBLP:conf/cav/AgarwalGKM22} and model-based reinforcement learning~\cite{Moerland2020modelbasedRL} commonly use optimistic exploration policies~\cite{DBLP:journals/ftml/Munos14}.
By contrast, we guide exploration based on the sensitivity analysis of the solution function with respect to the parametric model.
\section{Concluding Remarks}
\label{sec:conclusion}

We have presented efficient methods to compute partial derivatives of the solution functions for \glspl{pMC} and \glspl{prMC}.
For both models, we have shown how to compute these derivatives explicitly \emph{for all parameters}, as well as how to compute only the \emph{$k$ highest derivatives}.
Our experiments have shown that we can compute derivatives for models with over a million states and thousands of parameters.
In particular, computing the $k$ highest derivatives yields significant speed-ups compared to computing all derivatives explicitly and is feasible for prMCs which can be verified.
In the future, we want to support nondeterminism in the models and apply our methods in (online) learning frameworks, in particular for settings where reducing the uncertainty is computationally expensive~\cite{DBLP:conf/aips/NearyVCT22,DBLP:conf/cav/JungesS22}.

\bibliographystyle{splncs04}
\bibliography{references.bib}

\appendix

\section{Proofs and Mathematical Details}

\subsection{Proof of \cref{thm:derivative_prmc}}
\label{app:proof:thm2}

The proof follows from applying the product rule to the equation system in \cref{eq:prmc_LP_reduced}.
The derivative of the right-hand side \cref{eq:prmc_LP_red_constr1} (i.e., for all terminal states $s \in \terminalStates$) is trivially zero.
The derivative of \cref{eq:prmc_LP_red_constr2} is
\begin{equation}
\begin{split}
    \parder{x_s}{u(\param_i)} &= - \parder{\left(b_s[\U]^\top \diag{E_s} \alpha_s\right)}{u(\param_i)} - \parder{\beta_s}{u(\param_i)}
    \\
    \parder{x_s}{u(\param_i)} &= - (\alpha_s^\star)^\top \diag{E_s} \parder{b_s[\U]}{u(\param_i)} - b_s[\U]^\top \diag{E_s} \parder{\alpha}{u(\param_i)} - \parder{\beta_s}{u(\param_i)},
\end{split}
\end{equation}
which, after rearranging, yields \cref{eq:prmc_LP_deriv_constr2}.
Similarly, the derivative of \cref{eq:prmc_LP_red_constr3} is
\begin{equation}
    (\alpha_s^\star)^\top \diag{E_s} \parder{A_s[\U]}{u(\param_i)} + A_s[\U]^\top \diag{E_s} \parder{\alpha_s}{u(\param_i)} + \parder{x_{\post{s}}}{u(\param_i)} + \parder{\beta_s}{u(\param_i)} \mathbbm{1},
\end{equation}
which after rearranging yields \cref{eq:prmc_LP_deriv_constr3}, so we conclude the proof.

\subsection{Computing $k$-highest derivatives for prMCs}
\label{app:prmc_derivatives}

Analogous to \cref{eq:pMC_importance_MIP}, we can compute the $k \leq \ell = |\Param|$ highest derivatives of a \gls{prMC} based on the solution to a MILP.
For brevity, let us define the notations $x'_s = \parder{x_s}{u(\param_i)} \in \Real$, $\alpha'_s = \parder{\alpha_s}{u(\param_i)} \in \Real^{m_s}$ and $\beta'_s = \parder{\beta_s}{u(\param_i)} \in \Real$.
Using this notation, we obtain the following MILP:
\begin{subequations}
    \label{eq:prMC_importance_MIP}%
    \begin{alignat}{2}
    \maximize_{x', \alpha', \beta', z \in \{ 0, 1 \}^{\ell}} \,\, & \sI^\top x' &&
    \label{eq:prMC_importance_MIP:obj}
    \\
    \text{subject to} \,\, & x'_s = 0,
    && \qquad\forall s \in \terminalStates
    \label{eq:prMC_importance_MIP:cns1}
    \\
    & 
    x'_s + b_s[\U]^\top \diag{E_s} \alpha'_s + \beta'_s 
    &&= 
    -(\alpha_s^\star)^\top \diag{E_s} \sum_{i=1}^\ell z_i \parder{b_s[\U]}{u(\param_i)}, 
    \nonumber
    \\
    & 
    && 
    \qquad\forall s \in \States \setminus \terminalStates
    \label{eq:prMC_importance_MIP:cns2}
    \\    
    & 
    A_s[\U]^\top \diag{E_s} \alpha'_s + x'_{\post{s}} + \beta'_s \mathbbm{1}
    &&=
    -(\alpha^\star_s)^\top \diag{E_s} \sum_{i=1}^\ell z_i \parder{A_s[\U]}{u(\param_i)},
    \nonumber
    \\
    & 
    && 
    \qquad\forall s \in \States \setminus \terminalStates.
    \label{eq:prMC_importance_MIP:cns3}
    \\
    & z_1 + \cdots + z_{\ell} = k. &&
    \label{eq:prMC_importance_MIP:cns4}
    \end{alignat}%
\end{subequations}%
Observe that the difference between the constraints in \cref{eq:prMC_importance_MIP:cns2,eq:prMC_importance_MIP:cns3} and the equation system in \cref{eq:prmc_LP_deriv} lies in the summation over $i=1,\ldots,\ell$.
We derive the same LP relaxation as in \cref{eq:pMC_importance_LP}, i.e., we relax the binary variables $z \in \{0,1\}^\ell$ to continuous variables $z \in [0,1]^\ell$.
Since \cref{eq:prMC_importance_MIP} has the exact same characteristics as \cref{eq:pMC_importance_MIP}, \cref{thm:LP_relaxation} applies equivalently to the case for \glspl{prMC}.
In other words, the LP relaxation is exact, and we can use the resulting solution to find the set $V^\star$ of parameters with maximal derivatives using \cref{prop:pMC}.

{
\setlength{\tabcolsep}{3pt}
\begin{table*}[b!]

\centering
\caption{Model sizes, runtimes, and derivatives for all grid world benchmarks.}

\begin{threeparttable}

\scalebox{0.75}{
\setlength{\tabcolsep}{3pt}
\begin{tabular}{lrrrrrrrrrr}
\toprule
\multicolumn{4}{c}{{Model statistics}} 
& \multicolumn{2}{c}{{Verifying}} 
& \multicolumn{1}{c}{{Problem 1}} 
& \multicolumn{2}{c}{{Problem 3}} 
& \multicolumn{2}{c}{{Derivatives}} \\
  \cmidrule(lr){1-4} \cmidrule(lr){5-6} \cmidrule(lr){7-7} \cmidrule(lr){8-9} \cmidrule(lr){10-11}
Type & $|\States|$ & $|\Param|$ & \#trans & $\ensuremath{\mathsf{sol}_{(R)}}[\U]$ & Time [s] & All derivs. [s] & $k=1$ [s] & $k=10 [s]$ & Highest & Error \% \\
\midrule
pMC &     200 &         10 &         595 &     1.87 &             0.05 &                        0.11 &               0.15 &                0.16 &         5.53e-01 &          0.0 \\
pMC &     800 &        100 &        2395 &     7.91 &             0.24 &                        0.63 &               0.48 &                0.47 &         1.92e+00 &          0.0 \\
pMC &    5000 &         50 &       14995 &     5.07 &             1.39 &                        3.32 &               2.64 &                2.69 &         1.54e+00 &          0.0 \\
pMC &    5000 &        100 &       14995 &     5.05 &             1.36 &                        4.17 &               2.63 &                2.66 &         1.28e+00 &          0.0 \\
pMC &    5000 &        921 &       14995 &     4.93 &             1.87 &                       19.92 &               4.52 &                2.87 &         1.20e+00 &          0.0 \\
pMC &   20000 &        100 &       59995 &     1.99 &            14.22 &                       35.51 &              26.88 &               10.72 &         4.92e-01 &          0.0 \\
pMC &   20000 &       1000 &       59995 &     2.04 &             5.65 &                       97.23 &              11.68 &               12.41 &         5.18e-01 &          0.0 \\
pMC &   80000 &        100 &      239995 &     8.01 &            25.54 &                       98.47 &              45.18 &               46.87 &         1.95e+00 &          0.0 \\
pMC &   80000 &       1000 &      239995 &     8.01 &            25.64 &                      612.97 &              48.92 &               58.20 &         2.08e+00 &          0.0 \\
pMC &   80000 &       9831 &      239995 &     7.93 &            25.52 &                    5,650.25\tnote{a} &             347.76 &            1,343.59 &         2.10e+00 &          0.0 \\
pMC &  320000 &        100 &      959995 &     3.22 &           134.39 &                      659.06 &             223.71 &              231.73 &         8.69e-01 &          0.0 \\
pMC &  320000 &       1000 &      959995 &     3.20 &           133.25 &                    4,712.32\tnote{a} &             244.07 &              295.71 &         8.02e-01 &          0.0 \\
pMC &  320000 &      10000 &      959995 &     3.28 &           133.66 &                   45,655.90\tnote{a} &             447.62 &              831.01 &         9.08e-01 &          0.0 \\
pMC & 1280000 &        100 &     3839995 &    12.90 &           902.52 &                    4,747.43\tnote{a} &           1,396.51 &            1,507.77 &         3.32e+00 &          0.0 \\
pMC & 1280000 &       1000 &     3839995 &    12.79 &           902.67 &                   37,078.12\tnote{a} &           1,550.45 &            1,617.27 &         3.18e+00 &          0.0 \\
pMC & 1280000 &      10000 &     3839995 &   Timeout\tnote{b} --- &           --- &                   --- &           --- &            --- &         --- &          --- \\
\midrule
prMC &     200 &         10 &         595 &     3.36 &             0.93 &                        0.03 &               0.02 &                0.03 &        -3.02e-04 &         -0.0 \\
prMC &     800 &        100 &        2395 &    26.44 &             4.13 &                        0.49 &               0.09 &                0.11 &        -1.57e-03 &         -0.0 \\
prMC &    5000 &         50 &       14995 &   141.55 &            23.25 &                        1.67 &               0.59 &                0.66 &        -1.88e-02 &         -0.0 \\
prMC &    5000 &        100 &       14995 &   136.07 &            23.46 &                        3.55 &               0.60 &                1.58 &        -1.26e-02 &         -0.0 \\
prMC &    5000 &        921 &       14995 &   138.74 &            29.82 &                       25.23 &               0.85 &                1.09 &        -4.44e-03 &         -0.0 \\
prMC &   20000 &        100 &       59995 & 2,789.77 &         1,276.43 &                       15.68 &               2.40 &                2.70 &        -4.96e-01 &         -0.1 \\
prMC &   20000 &       1000 &       59995 & 2,258.41 &           339.96 &                      159.70 &               3.53 &                4.09 &        -9.51e-02 &         -0.0 \\
prMC &   80000 &       100 &       239995 & Timeout\tnote{b} &           --- &                      --- &               --- &                --- &        --- &         --- \\
prMC &   80000 &      1000 &       239995 & Timeout\tnote{b} &           --- &                      --- &               --- &                --- &        --- &         --- \\
prMC &   80000 &     10000 &       239995 & Timeout\tnote{b} &           --- &                      --- &               --- &                --- &        --- &         --- \\
\bottomrule
\end{tabular}

}
\begin{tablenotes}
        \raggedright
        \item[a] \tableextrapolationtext
        \item[b] \tabletimeouttext
\end{tablenotes}
\end{threeparttable}

\label{tab:grid_world_results_full}
\end{table*}
}

\section{Detailed Benchmark Results}
\label{app:experiments}

We provide a more detailed overview of the setup and the results for the numerical experiments performed in \cref{sec:experiments}.
The complete overview of the statistics of the grid world benchmarks is shown in \cref{tab:grid_world_results_full}.
Moreover, the complete table for all other benchmarks is presented in \cref{tab:benchmarks}.

\subsection{Details about the learning framework}
\label{app:experiments:learning}
As described in \cref{sec:experiments}, we use our methods to determine where to sample in two benchmarks.
First, we consider a slippery grid world with $|\States| = 800$ states and $|\Param| = 100$ terrain types (to bias the importance towards specific terrain types, we distribute $10$ terrain types over $50\%$ of the states and the remaining terrains over the other $50\%$).
Second, we consider the drone motion planning problem from \cite{DBLP:journals/tac/CubuktepeJJKT22} with $|\States| = 4\,179$ states and $|\Param| = 1\,053$ parameters.
For both benchmarks, we compare four sampling strategies:
\begin{enumerate}
    \item greedy sampling for the highest derivative, i.e., solving Problem~3 for $k=1$;
    \item greedy sampling for the biggest interval width $\epsilon_i$ (as we use the same confidence level $\beta$ on each interval, note that this strategy coincides with choosing the parameter with the lowest sample size);
    \item uniform sampling over all parameters;
    \item weighted sampling according to an importance factor for each parameter. Specifically, inspired by \cite{DBLP:conf/cav/JungesS22}, we define the importance of each parameter $i = \{1,\ldots,|\Param|\}$ as the \emph{expected number of visits of states in which parameter $i$ appears} multiplied with the \emph{width $\epsilon_i$ of the outgoing probability interval}.
\end{enumerate}
As discussed in \cref{sec:experiments}, for each parameter $\param_i \in \Param$ we compute a probability interval using Hoeffding's inequality, resulting in a symmetric probability interval around the sample mean $\hat{p}_i$.
To apply our methods, we then need to differentiate the interval bounds with respect to the sample size $N_i$ for each parameter $\param_i \in \Param$.
While differentiating $\epsilon_i = \sqrt{\frac{\log{2} - \log{(1-\beta)}}{2N}}$ with respect to $N_i$ is straightforward, the change in the sample mean $\hat{p}_i$ depends on the outcome of the additional samples.
To avoid this issue, we assume that the sample mean $\hat{p}_i$ remains constant, which is a reasonable assumption, especially for higher sample sizes.
Despite the fact that we use this heuristic, \cref{fig:learning} shows that our derivative-guided sampling strategy converges to the true solution faster than the other strategies.

{
\setlength{\tabcolsep}{3pt}
\begin{table*}[t!]

\centering
\caption{Model sizes and runtimes for all benchmarks other than grid worlds. We only solve Problem 3 for $k=10$ parameters for models with at least $10$ parameters, i.e., $|\Param| \geq 10$.}

\begin{threeparttable}

\scalebox{0.8}{
\setlength{\tabcolsep}{3pt}
\begin{tabular}{llrrrrrrrr}
\toprule
\multicolumn{5}{c}{{Model statistics}} & \multicolumn{2}{c}{{Verification}} & \multicolumn{1}{c}{{All derivs.}} & \multicolumn{2}{c}{{$k$ highest derivs.}} \\
  \cmidrule(lr){1-5} \cmidrule(lr){6-7} \cmidrule(lr){8-8} \cmidrule(lr){9-10}
Instance & Type & $|\States|$ & $|\Param|$ & \#trans & $\ensuremath{\mathsf{sol}_{(R)}}[\U]$ & Verify [s] & Explicit [s] & $k=1$ [s] & $k=10$ [s] \\
\midrule
BRP (16,2) &  pMC &    613 &          2 &         803 &          0.10 &             0.11 &                        0.17 &               0.20 &                 --- \\
BRP (32,3) &  pMC &   1638 &          2 &        2179 &          0.04 &             0.29 &                        0.46 &               0.51 &                 --- \\
BRP (64,4) &  pMC &   4103 &          2 &        5507 &          0.02 &             0.74 &                        1.13 &               1.18 &                 --- \\
BRP (512,5) &  pMC &  39432 &          2 &       53251 &          0.02 &             7.77 &                       11.66 &              12.12 &                 --- \\
BRP (1024,6) &  pMC &  92169 &          2 &      124931 &          0.01 &            20.66 &                       29.69 &              31.76 &                 --- \\
Crowds (3,5) &  pMC &   1367 &          2 &        2027 &          0.87 &             0.25 &                        0.36 &               0.40 &                 --- \\
Crowds (6,5) &  pMC &     14 &          1 &          16 &          0.25 &             0.00 &                        0.01 &               0.03 &                 --- \\
Crowds (10,5) &  pMC & 104512 &          2 &      246082 &          0.83 &            27.86 &                       36.02 &              36.28 &                 --- \\
NAND (2,4) &  pMC &    326 &          2 &         435 &          0.71 &             0.06 &                        0.10 &               0.14 &                 --- \\
NAND (5,10) &  pMC &   8112 &          2 &       11577 &          0.57 &             1.60 &                        2.35 &               2.42 &                 --- \\
NAND (10,15) &  pMC & 104412 &          2 &      156247 &          0.51 &            26.37 &                       36.55 &              37.02 &                 --- \\
   Virus &  pMC &    761 &         14 &        5009 &         88.89 &             0.18 &                        0.57 &               0.60 &                0.60 \\
   WLAN0 &  pMC &   2711 &         15 &        4877 &     65,886.52 &             0.47 &                        1.02 &               1.06 &                1.07 \\
CSMA (2,4) &  pMC &   7958 &         26 &       10594 &         98.58 &             1.39 &                        3.05 &               3.07 &                3.09 \\
Coin (4) &  pMC &  22656 &          4 &       74957 &          0.92 &             4.25 &                        7.71 &               7.80 &                 --- \\
    Maze &  pMC &   1303 &        590 &        2658 &         79.81 &             0.26 &                        1.76 &               0.54 &                0.54 \\
Drone (mem1) &  pMC &   4179 &       1053 &        9414 &          0.11 &             0.88 &                       12.34 &               1.44 &                1.52 \\
Drone (mem5) &  pMC &  32403 &      12286 &       70099 &          0.11 &             8.58 &                   16,592.78\tnote{a} &              12.73 &               14.02 \\
Satellite (36,5) &  pMC &  31325 &       2555 &      156924 &          0.00 &            10.68 &                      402.14 &              65.83 &               66.30 \\
Satellite (36,65) &  pMC & 217561 &      10042 &      615433 &          0.00 &            96.48 &                  203,804.52\tnote{a} &             179.62 &              199.52 \\
\midrule
BRP (16,2) & prMC &    613 &          2 &         803 &          0.24 &             1.17 &                        0.02 &               0.03 &                 --- \\
BRP (16,2) & prMC &    613 &        190 &         803 &          0.24 &             1.16 &                        0.19 &               0.03 &                0.03 \\
BRP (32,3) & prMC &   1638 &          2 &        2179 &          0.13 &             3.09 &                        0.04 &               0.05 &                 --- \\
BRP (32,3) & prMC &   1638 &        541 &        2179 &          0.13 &             3.10 &                        1.00 &               0.05 &                0.06 \\
BRP (64,4) & prMC &   4103 &          2 &        5507 &          0.07 &             7.97 &                        0.11 &               0.12 &                 --- \\
BRP (64,4) & prMC &   4103 &       1404 &        5507 &          0.07 &             8.00 &                        5.75 &               0.13 &                0.17 \\
BRP (512,5) & prMC &  39432 &          2 &       53251 &          0.14 &            76.16 &                        1.11 &               1.69 &                 --- \\
BRP (512,5) & prMC &  39432 &      13819 &       53251 &          0.14 &            76.34 &                      746.47 &               1.71 &                1.59 \\
BRP (1024,6) & prMC &  92169 &          2 &      124931 &          0.07 &           181.33 &                        2.66 &               5.08 &                 --- \\
BRP (1024,6) & prMC &  92169 &      32762 &      124931 &          0.08 &           180.37 &                    4,625.87\tnote{a} &               5.27 &                3.33 \\
Crowds (3,5) & prMC &   1367 &          2 &        2027 &          0.91 &             2.28 &                        0.03 &               0.04 &                 --- \\
Crowds (3,5) & prMC &   1367 &        495 &        2027 &          0.92 &             2.84 &                        1.68 &               0.05 &                0.42 \\
Crowds (6,5) & prMC &     14 &          2 &          16 &          0.29 &             0.02 &                        0.00 &               0.00 &                 --- \\
Crowds (10,5) & prMC & 104512 &          2 &      246082 &          0.88 &           183.58 &                        4.12 &               2.95 &                 --- \\
Crowds (10,5) & prMC & 104512 &      38610 &      246082 &          0.88 &           389.26 &                   36,477.85\tnote{a} &               5.19 &                 --- \\
NAND (2,4) & prMC &    326 &          2 &         435 &          0.94 &             0.60 &                        0.01 &               0.01 &                 --- \\
NAND (2,4) & prMC &    326 &        109 &         435 &          0.95 &             0.72 &                        0.09 &               0.02 &                0.07 \\
NAND (5,10) & prMC &   8112 &          2 &       11577 &          0.99 &            15.91 &                        0.38 &               0.22 &                 --- \\
NAND (5,10) & prMC &   8112 &       3465 &       11577 &          1.00 &            18.27 &                      158.02 &               0.41 &                2.82 \\
NAND (10,15) & prMC & 104412 &          2 &      156247 &          1.00 &           341.76 &                       37.58 &               3.62 &                 --- \\
NAND (10,15) & prMC & 104412 &      51835 &      156247 &          1.00 &           764.31 &                   44,289.51\tnote{a} &               5.79 &                 --- \\
   Virus & prMC &    761 &         14 &        5009 &        151.30 &             2.26 &                        0.06 &               0.04 &                0.31 \\
   Virus & prMC &    761 &        558 &        5009 &        151.53 &             2.29 &                        0.85 &               0.04 &                0.05 \\
   WLAN0 & prMC &   2711 &         45 &        4877 & \num{66.7e06} &             3.13 &                        0.18 &               0.04 &                0.55 \\
CSMA (2,4) & prMC &   7958 &          1 &       10594 &        100.11 &             7.56 &                        0.04 &               0.07 &                 --- \\
CSMA (2,4) & prMC &   7958 &        202 &       10594 &        113.48 &             9.80 &                        1.68 &               0.18 &                0.16 \\
Coin (4) & prMC &  22656 &          4 &       74957 &          0.97 &            39.66 &                        1.00 &               0.69 &                 --- \\
Coin (4) & prMC &  22656 &       6214 &       74957 &          0.97 &            38.56 &                      876.16 &               0.71 &                6.49 \\
    Maze & prMC &   1303 &        590 &        2658 &        177.80 &             5.11 &                        2.36 &               0.13 &                0.14 \\
    Maze & prMC &   1303 &       1300 &        2658 &        207.22 &             5.14 &                       10.78 &               0.13 &                0.93 \\
Drone (mem1) & prMC &   4179 &       1053 &        9414 &          0.20 &             9.54 &                       19.30 &               0.23 &                0.34 \\
Drone (mem1) & prMC &   4179 &       4176 &        9414 &          0.39 &            17.64 &                    1,018.50 &               0.82 &                7.63 \\
Drone (mem5) & prMC &  32403 &      12286 &       70099 &          0.20 &           844.53 &                   10,800.88 &               2.95 &                4.59 \\
Drone (mem5) & prMC &  32403 &      32401 &       70099 &          0.39 &         1,808.07 &                  173,662.85\tnote{a} &               4.51 &                 --- \\
Satellite (36,5) & prMC &  31325 &         50 &      156924 &          0.00 &            43.04 &                       83.81 &               0.95 &               60.26 \\
Satellite (36,5) & prMC &  31325 &      30040 &      156924 &          0.12 &           609.86 &                  186,404.70\tnote{a} &               4.64 &              232.28 \\
Satellite (36,65) &  prMC & 217561 &      10042 &      615433 &          Timeout\tnote{b} &            --- &                  --- &             --- &              --- \\
\bottomrule
\end{tabular}
}
\begin{tablenotes}
        \raggedright
        \item[a] \tableextrapolationtext
        \item[b] \tabletimeouttext
\end{tablenotes}
\end{threeparttable}
\label{tab:benchmarks}
\end{table*}
}

\end{document}